%% file: paper.tex
\documentclass[12pt]{article}

\usepackage{fancyhdr,amsmath,amsthm,amssymb,bm} 
\usepackage{graphicx,caption}
\usepackage{float}
\usepackage{bm}
\usepackage{bbm}
\usepackage{amsfonts,verbatim,url}
\usepackage{dsfont}
\usepackage[top=1in, bottom=1in, left=1in, right=1in]{geometry}
\usepackage[square,numbers]{natbib}
\usepackage{color}
\usepackage{setspace} 
\usepackage{algorithm}
\usepackage{algorithmic}
\usepackage{blindtext}
\usepackage{mathtools}
\usepackage{verbatim}
\usepackage{textcomp}
\usepackage{authblk}

\usepackage[usenames,dvipsnames,svgnames,table]{xcolor}
\usepackage[colorlinks=true,
            linkcolor=blue,
            urlcolor=blue,
            citecolor=blue]{hyperref}

\let\tilde\widetilde
\let\hat\widehat

\usepackage{mkolar_definitions}
\graphicspath{{fig/}}
\mathtoolsset{showonlyrefs}

\begin{document}

\title{Recovery of simultaneous low rank and two-way sparse
  coefficient matrices, a nonconvex approach} 

\author{Ming Yu\thanks{Booth School of Business, The University of Chicago. Email: \href{mailto:mingyu@chicagobooth.edu}{mingyu@chicagobooth.edu} } , 
Varun Gupta\thanks{Booth School of Business, The University of Chicago. Email: \href{mailto:varun.gupta@chicagobooth.edu}{varun.gupta@chicagobooth.edu}} , 
and Mladen Kolar\thanks{Booth School of Business, The University of Chicago. Email: \href{mailto:mladen.kolar@chicagobooth.edu}{mladen.kolar@chicagobooth.edu}}}

\date{}
\maketitle

\begin{abstract}
  We study the problem of recovery of matrices that are simultaneously
  low rank and row and/or column sparse. Such matrices appear in
  recent applications in cognitive neuroscience, imaging, computer
  vision, macroeconomics, and genetics.  We propose a GDT (Gradient
  Descent with hard Thresholding) algorithm to efficiently recover
  matrices with such structure, by minimizing a bi-convex function
  over a nonconvex set of constraints. We show linear convergence of
  the iterates obtained by GDT to a region within statistical error of
  an optimal solution. As an application of our method, we consider
  multi-task learning problems and show that the statistical error
  rate obtained by GDT is near optimal compared to minimax rate. Experiments
  demonstrate competitive performance and much faster running speed
  compared to existing methods, on both simulations and real data sets.


\end{abstract}


\input{Introduction}

\input{Formulation}

\input{Theoretical}

\input{Statistical_error}

\input{MTRL}

\input{Experiment}

\input{Appendix.tex}

\section{Conclusion}
\label{sec:conclusion}

We proposed a new GDT algorithm to efficiently solve for optimization
problem with simultaneous low rank and row and/or column sparsity
structure on the coefficient matrix.  We show the linear convergence of 
GDT algorithm up to statistical error.  As an application, for
multi-task learning problem we show that the statistical error is near
optimal compared to the minimax rate.  Experiments on multi-task
learning demonstrate competitive performance and much faster running
speed compared to existing methods.  For future extensions, it would
be of interest to extend GDT algorithm to non-linear models.  Another
potential direction would be to adaptively select the sparsity level
$s_1$ and $s_2$ in hard thresholding step.

\section*{Acknowledgments}

This work was completed in part with resources provided by the
University of Chicago Research Computing Center. We thank Zhaoran Wang and
Zhuoran Yang for many useful discussions and for suggesting
an application to multi-task reinforcement learning.

\bibliographystyle{plain}
\bibliography{paper} 

\end{document}

%% file: Introduction.tex
\section{Introduction}
\label{sec:introduction}

Many problems in machine learning, statistics and signal processing
can be formulated as optimization problems with a smooth objective and
nonconvex constraints. The objective usually measures the fit of a
model, parameter, or signal to the data, while the constraints encode
structural requirements on the model.  Examples of nonconvex
constraints include sparsity where the parameter is assumed to have
only a few non-zero coordinates \cite{Hsu2011Robust,
  Yu2016Statistical, She2017Selective, turlach2005simultaneous,
  Zhu2016Personalized}, group sparsity where the parameter is
comprised of several groups only few of which are non-zero
\cite{Lounici2011Oracle, kim2010tree, Huang2010benefit,
  Chen2011Integrating}, and low-rankness where the parameter is
believed to be a linear combination of few factors
\cite{Amit2007Uncovering, Chen2011Reduced, Chen2015Fast,
  Gross2011Recovering, Jain2013Low}. Common approach to dealing with
nonconvex constraints is via convex relaxations, which allow for
application of simple optimization algorithms and easy theoretical
analysis \cite{Agarwal2012Noisy, Candes2010power, Fazel2001rank,
  Candes2009Exact, Koltchinskii2011Nuclear}. From a practical point
of view, it has been observed that directly working with a nonconvex
optimization problem can lead to both faster and more accurate
algorithms \cite{Sun2016Guaranteed, Zhao2015Nonconvex,
  yu2017influence, wang2014nonconvex}. As a result, a body of
literature has recently emerged that tries to characterize good
performance of these algorithms \cite{Barber2017Gradient,
  Zhang2017Nonconvex, Ha2017Alternating}.

In this work, we focus on the following optimization problem
\begin{equation}
  \label{eq:opt}
  \hat \Theta \in \arg \min_{\Theta \in \Xi} f(\Theta)
\end{equation}
where $\Xi \subset \RR^{m_1 \times m_2}$ is a nonconvex set comprising of
low rank matrices that are also row and/or column sparse,
\[
\Xi = \Xi(r,s_1,s_2) = \{ \Theta \in \RR^{m_1 \times m_2} \mid {\rm rank}(\Theta) \leq r,
\|\Theta\|_{2,0} \leq s_1, \|\Theta^\top\|_{2,0} \leq s_2\},
\]
where
$\|\Theta\|_{2,0} = |\{ i \in [m_1] \mid \sum_{j \in [m_2]}
\Theta_{ij}^2 \neq 0\}|$
is the number of non-zero rows of $\Theta$. Such an optimization
problem arises in a number of applications including sparse singular
value decomposition and principal component analysis
\cite{wang2014nonconvex, Ma2014Learning, Hastie2015Matrix}, sparse
reduced-rank regression \cite{Bunea2012Joint, Ma2014Adaptive,
  Chen2011Reduced, Chen2012Sparse, Vounou2012Sparse}, and
reinforcement learning \cite{calandriello2014sparse,
  sutton1998introduction, Lazaric2010Bayesian, wilson2007multi,
  Snel2012Multi}.  Rather than considering convex relaxations of the
optimization problem \eqref{eq:opt}, we directly work with a nonconvex
formulation.  Under an appropriate statistical model, the global
minimizer $\hat \Theta$ approximates the ``true'' parameter $\Theta^*$
with an error level $\epsilon$. Since the optimization problem
\eqref{eq:opt} is highly nonconvex, our aim is to develop an iterative
algorithm that, with appropriate initialization, converges linearly to
a stationary point $\check \Theta$ that is within $c\cdot\epsilon$
distance of $\hat \Theta$. In order to develop a computationally
efficient algorithm, we reparametrize the $m_1 \times m_2$ matrix
variable $\Theta$ as $UV^\top$ with $U \in \RR^{m_1 \times r}$ and
$V \in \RR^{m_2 \times r}$, and optimize over $U$ and $V$. That is, we
consider (with some abuse of notation) the following optimization
problem
\begin{equation}
  \label{eq:opt:rep}
\begin{aligned}
(\hat U, \hat V) \in \arg \min_{U \in \Ucal , V \in \Vcal} 
f(U,V),
\end{aligned}
\end{equation}
where 
\[
\Ucal = \Ucal(s_1) = \cbr{ U \in \RR^{m_1 \times  r} \mid \|U\|_{2,0} \leq s_1 }
\quad\text{and}\quad
\Vcal = \Vcal(s_2) = \cbr{ V \in \RR^{m_2 \times  r} \mid \|V\|_{2,0} \leq s_2 }.
\] 
Such a reparametrization automatically enforces the low rank structure
and will allow us to develop an algorithm with low computational cost
per iteration. Note that even though $\hat U$ and $\hat V$ are only
unique up to scaling and a rotation by an orthogonal matrix,
$\hat \Theta = \hat U\hat V^\top$ is usually unique.

We make several contributions in this paper. First, we develop an
efficient algorithm for minimizing \eqref{eq:opt:rep}, which uses
projected gradient descent on a nonconvex set in each iteration. Under
conditions on the function $f(\Theta)$ that are common in the
high-dimensional literature, we establish linear convergence of the
iterates to a statistically relevant solution.  In particular, we
require that the function $f(\Theta)$ satisfies restricted strong
convexity (RSC) and restricted strong smoothness (RSS), conditions
that are given in Condition ({\bf RSC/RSS}) below.  Compared to the
existing work for optimization over low rank matrices with
(alternating) gradient descent, we need to study a projection onto a
nonconvex set in each iteration, which in our case is a
hard-thresholding operation, that requires delicate analysis and novel
theory. Our second contribution, is in the domain of multi-task
learning. Multi-task learning is a widely used learning framework
where similar tasks are considered jointly for the purpose of
improving performance compared to learning the tasks separately
\cite{caruana1997multitask}. We study the setting where the number of
input variables and the number of tasks can be much larger than the
sample size (see \cite{Ma2014Adaptive} and references there in). Our
focus is on simultaneous variable selection and dimensionality
reduction. We want to identify which variables are relevant predictor
variables for different tasks and at the same time we want to combine
the relevant predictor variables into fewer features that can be
explained as latent factors that drive the variation in the multiple
responses.  We provide a new algorithm for this problem and improve
the theoretical results established in \cite{Ma2014Adaptive}. In
particular, our algorithm does not require a new independent sample in
each iteration and allows for non-Gaussian errors, while at the same
time achieves nearly optimal error rate compared to the information
theoretic minimax lower bound for the problem. Moreover, our
prediction error is much better than the error bound proposed in
\cite{Bunea2012Joint}, and matches the error bound in
\cite{She2017Selective}. However, all of the existing algorithms are
slow and cannot scale to high dimensions.  Finally, our third
contribution is in the area of reinforcement learning. We study the
Multi-task Reinforcement Learning (MTRL) problem via value function
approximation. In MTRL the decision maker needs to solve a sequence of
Markov Decision Processes (MDPs). A common approach to Reinforcement
Learning when the state space is large is to approximate the value
function of linear basis functions (linear in some appropriate feature
representation of the states) with sparse support. Thus, it is natural
to assume the resulting coefficient matrix is low rank and row
sparse. Our proposed algorithm can be applied to the regression step
of any MTRL algorithm (we chose Fitted $Q$-iteration (F$Q$I) for
presentation purposes) to solve for the optimal policies for
MDPs. Compared to \cite{calandriello2014sparse} which uses convex
relaxation, our algorithm is much more efficient in high dimensions.

\subsection{Related Work}

Our work contributes to several different areas, and thus is naturally
related to many existing works. We provide a brief overview of the
related literature and describe how it is related to our
contributions. For the sake of brevity, we do not provide an extensive
review of the existing literature.

{\bf Low-rank Matrix Recovery.} A large body of literature exists on
recovery of low-rank matrices as they arise in a wide variety of
applications throughout science and engineering, ranging from quantum
tomography to signal processing and machine learning 
\cite{Aaronson2007learnability,Liu2009Interior,Srebro2005Maximum,Davenport2016Overview}.
Recovery of a low-rank matrix can be formulated as the following optimization
problem
\begin{equation}
  \label{eq:opt:related}
  \hat \Theta \in \arg \min_{\Theta \in \RR^{m_1 \times m_2}} f(\Theta)
\quad \text{subject to } {\rm rank}(\Theta) \leq r,
\end{equation}
where the objective function $f : \RR^{m_1 \times m_2} \mapsto \RR$ is
convex and smooth. The problem \eqref{eq:opt:related} is highly
nonconvex and NP-hard in general \cite{Fazel2001rank,Fazel2004Rank}. A
lot of the progress in the literature has focused on convex
relaxations where one replaces the rank constraint using the nuclear
norm.  See, for example, \cite{Candes2009Exact, Candes2010power,
  Candes2010Matrix, Recht2010Guaranteed, Cai2010singular,
  Recht2011simpler, Gross2011Recovering, Chandrasekaran2011Rank,
  Hsu2011Robust, Rohde2011Estimation, Koltchinskii2011Nuclear,
  Harchaoui2012Large, Negahban2011Estimation, Chen2011Integrating,
  Xiang2012Optimal, Negahban2012Restricted, Agarwal2012Noisy,
  Recht2013Parallel, Chen2015Incoherence, Chen2014Coherent,
  Chen2013Low, Hastie2015Matrix, Cai2015ROP:, Yan2015Simultaneous,
  Zhu2016Personalized, Wang2015Orthogonal} and references therein.
However, developing efficient algorithms for solving these convex
relaxations is challenging in regimes with large $m_1$ and $m_2$
\cite{Hsieh2014Nuclear}.  A practical approach, widely used in large
scale applications such as recommendation systems or collaborative
filtering \cite{Takacs2007Major, Koren2009Matrix, Gemulla2011Large,
  Zhuang2013fast} relies on solving a nonconvex optimization problem
where the decision variable $\Theta$ is factored as $UV^\top$, usually
referred to as the Burer-Monteiro type decomposition
\cite{Burer2003nonlinear, Burer2005Local}. A stationary point of this
nonconvex problem is usually found via a block coordinate descent-type
algorithm, such as alternating minimization or (alternating) gradient
descent. Unlike for the convex relaxation approaches, the theoretical
understanding of these nonconvex optimization procedures has been
developed only recently \cite{Keshavan2010Matrix, Keshavan2010Matrixa,
  Jain2013Low, Hardt2014Understanding, Hardt2014Fast,
  Hardt2014Computational, Sun2016Guaranteed, Zhao2015Nonconvex,
  Zheng2015Convergent, Bhojanapalli2016Global,
  Bhojanapalli2016Dropping, Tu2016Low, Chen2015Fast, Zhu2017Globala,
  Zhu2017Global, Ge2016matrix, Li2017Geometry, Mei2016Landscape}.  Compared to the classical nonconvex
optimization theory, which only shows a sublinear convergence to a
local optima, the focus of the recent literature is on establishing
linear rates of convergence or characterizing that the objective
does not have spurious local minima. In addition to the methods that
work on the factorized form, \cite{Jain2010Guaranteed, Lee2010ADMiRA:,
  Jain2015Fast, Barber2017Gradient} consider projected gradient-type
methods which optimize over the matrix variable
$\Theta \in \RR^{m_1 \times m_2}$. These methods involve calculating
the top $r$ singular vectors of an $m_1 \times m_2$ matrix at each
iteration. When $r$ is much smaller than $m_1$ and $m_2$, they incur
much higher computational cost per iteration than the methods that
optimize over $U \in \RR^{m_1 \times r}$ and
$V \in \RR^{m_2 \times r}$.

Our work contributes to this body of literature by studying
gradient descent with a projection step on a non-convex
set, which requires hard-thresholding. Hard-thresholding in this
context has not been considered before. Theoretically we need a new
argument to establish linear convergence to a statistically
relevant point. \cite{Chen2015Fast} considered projected gradient
descent in a symmetric and positive semidefinite setting with a
projection on a convex set. 
Our work is most closely related to
\cite{Zhao2015Nonconvex}, which used the notion of inexact first order
oracle to establish their results, but did not consider the
hard-thresholding step.

{\bf Structured Low-rank Matrices.} Low-rank matrices with additional
structure also commonly arise in different problems ranging from
sparse principal component analysis (PCA) and sparse singular value
decomposition to multi-task learning. In a high-dimensional setting,
the classical PCA is inconsistent \cite{johnstone2009consistency} and
recent work has focused on PCA with additional sparse structure on the
eigenvectors \cite{Amini2009High, Berthet2013Optimal,
  Birnbaum2013Minimax, Cai2013Sparse, Vu2013Minimax, Ma2013Sparse,
  Yuan2013Truncated}. Similar sparse structure in singular vectors
arises in sparse SVD and biclustering \cite{Lee2010Biclustering,
  Chen2011Reduced, Ma2014Learning, Uematsu2017SOFAR:, Yang2014Sparse,
  Balakrishnan2012Recovering, Kolar2011Minimax,
  balakrishnan2011statistical}.  While the above papers use the
sparsity structure of the eigenvectors and singular vectors, it is
also possible to have simultaneous low rank and sparse structure
directly on the matrix $\Theta$. Such a structure arises in multi-task
learning, covariance estimation, graph denoising and link prediction
\cite{Mei2012Encoding, Richard2012Estimation}. Additional structure on
the sparsity pattern was imposed in the context of sparse rank-reduced
regression, which is an instance of multi-task learning
\cite{Chen2012Sparse, Bunea2012Joint, Ma2014Adaptive,
  Bahadori2016Scalable, She2017Selective}. Our algorithm described in
Section~\ref{sec:methodology} can be applied to the above mentioned
problems. In Section~\ref{sec:MTL}, we theoretically study multi-task
learning in the setting of \cite{Ma2014Adaptive}.  We relax conditions
imposed in \cite{Ma2014Adaptive}, specifically allowing for
non-Gaussian errors and not requiring independent samples at each step
of the algorithm, while still achieving the near minimax rate of
convergence. We provide additional discussion in Section~\ref{sec:MTL}
after formally providing results for the multi-task learning setting.
In Section~\ref{sec:experiment}, we further corroborate our theoretical
results in extensive simulations and show that our algorithm
outperforms existing methods in multi-task learning.

{\bf Low-rank Plus Sparse Matrix Recovery.} At this point, it is worth
mentioning another commonly encountered structure on the decision
variable $\Theta$ that we do not study in the current paper. In
various applications it is common to model $\Theta$ as a sum of two
matrices, one of which is low-rank and the other one
sparse. Applications include robust PCA, latent Gaussian graphical
models, factor analysis and multi-task learning
\cite{Candes2011Robust, Hsu2011Robust, Chandrasekaran2011Rank,
  Chen2013Low, Agarwal2012Noisy, Gu2016Low, Zhang2017Nonconvex,
  Xu2017Speeding, Ha2017Alternating}. While Burer-Monteiro
factorization has been considered for the low-rank component in this
context (see, for example, \cite{Zhang2017Nonconvex} and references
therein), the low-rank component is dense as it needs to be
incoherent.  The incoherence assumption guarantees that the low-rank
component is not too spiky and can be identified
\cite{Candes2009Exact}. An alternative approach was taken in
\cite{Ha2017Alternating} where alternating minimization over the
low-rank and sparse component with a projection on a nonconvex set was
investigated.

\subsection{Organization of the paper}

In Section \ref{sec:methodology} we provide details for our proposed
algorithm. Section \ref{sec:theoretical} states our assumptions and
the theoretical result with a proof sketch. Section \ref{sec:MTL}
shows applications to multi-task learning, while Section
\ref{sec:experiment} presents experimental results. Section
\ref{sec:appendix} provides detailed technical proofs. Conclusion is
given in Section~\ref{sec:conclusion}.


%% file: Formulation.tex
\section{Gradient Descent With Hard Thresholding}
\label{sec:methodology}

In this section, we detail our proposed algorithm, which is based on
gradient descent with hard thresholding (GDT).  Our focus is on
developing an efficient algorithm for minimizing $f(\Theta)$ with
$\Theta \in \Xi$. In statistical estimation and machine learning a
common goal is to find $\Theta^*$, which is an (approximate) minimizer
of $\EE[f(\Theta)]$ where the expectation is with respect to
randomness in data.  In many settings, the global minimizer of
\eqref{eq:opt} can be shown to approximate $\Theta^*$ up to
statistical error, which is problem specific.  In
Section~\ref{sec:theoretical}, we will show that iterates of our
algorithm converge linearly to $\Theta^*$ up to a statistical
error. It is worth noting that an argument similar to that in the
proof of Theorem~\ref{main} can be used to establish linear
convergence to the global minimizer $\hat \Theta$ in a deterministic
setting. That is, suppose $(\hat U, \hat V)$ is a global minimizer of
the problem \eqref{eq:opt:rep} and
$\hat \Theta = \hat U \hat V^{\top}$.  Then as long as the conditions
in Section~\ref{sec:theoretical} hold for $\hat U, \hat V$ in place of
$U^*, V^*$, we can show linear convergence to $\hat \Theta$ up to an
error level defined by the gradient of the objective function at
$\hat \Theta$. See the discussion after Theorem~\ref{main}.

Our algorithm, GDT, uses a Burer-Monteiro factorization to write
$\Theta = UV^{\top}$, where $U \in \R^{m_1 \times r}$ and
$V \in \R^{m_2 \times r}$, and minimizes 
\begin{equation}
  \label{eq:opt:rep:1}
\begin{aligned}
(\hat U, \hat V) \in \arg \min_{U \in \Ucal , V \in \Vcal} 
f(U,V) + g(U,V),
\end{aligned}
\end{equation}
where $g(U,V)$ is the penalty function defined as
\begin{equation}
\label{eq:penalty}
g(U,V) = \frac 14 \|U^{\top}U - V^{\top}V \|_F^2.
\end{equation}
The role of the penalty is to find a balanced decomposition of
$\hat \Theta$, one for which $\sigma_i(\hat U) = \sigma_i(\hat V)$,
$i=1,\ldots,r$ \cite{Zhu2017Global, Zhang2017Nonconvex}.  Note the
value of the penalty is equal to $0$ for a balanced solution, so we
can think of the penalized objective as looking through minimizer of
\eqref{eq:opt:rep} for a one that satisfies
$\hat U^{\top}\hat U - \hat V^{\top}\hat V = 0$.  In particular,
adding the penalty function $g$ does not change the minimizer of $f$
over $\Xi$. The convergence rate of GDT depends on the condition
number of $(U^*,V^*)$, the point algorithm converges to. The penalty
ensures that the iterates $U,V$ are not ill-conditioned.  Gradient
descent with hard-thresholding on $U$ and $V$ is used to minimize
\eqref{eq:opt:rep:1}. Details of GDT are given in
Algorithm~\ref{algo:AltGD}.  The algorithm takes as input parameters
$\eta$, the step size; $s_1$, $s_2$, the sparsity level; $T$, the
number of iterations; and a starting point $\Theta^0$.

The choice of starting point $\Theta^0$ is very important as the
algorithm performs a local search in its neighborhood. In
Section~\ref{sec:theoretical} we will formalize how close $\Theta^0$
needs to be to $\Theta^*$, while in Section~\ref{sec:MTL} we provide a
concrete way to initialize under a multi-task learning model. In
general, we envisage finding $\Theta^0$ by solving the following
optimization problem
\begin{equation}
  \label{eq:meta_init}
  \Theta^0 = \arg\min_{\Theta \in \RR^{m_1 \times m_2}} f(\Theta) + {\rm pen}(\Theta),
\end{equation}
where ${\rm pen}(\Theta)$ is a (simple) convex penalty term making the
objective \eqref{eq:meta_init} a convex optimization problem.  For
example, we could use the vector $\ell_1$ norm,
${\rm pen}(\Theta) = \|\Theta\|_1$. 
The choice of
penalty ${\rm pen}(\Theta)$ should be such that solving the
optimization problem in \eqref{eq:meta_init} can be done efficiently
in a high dimensional setting.  In practice, if solving the convex
relaxation is slow, we can start from the all zero matrix and perform
several (proximal) gradient steps to get an appropriate
initialization. See for example \cite{Zhang2017Nonconvex}.  Once an
initial estimate $\Theta^0$ is obtained, we find the best rank $r$
approximation $\tilde\Theta = \tilde U \tilde \Sigma \tilde V^\top$ to
$\Theta^0$ and use it to obtain the initial iterates $ U^0$ and $V^0$.
In each step, GDT updates $U$ and $V$ by taking a gradient step and
hard-thresholding the result. The operation $\text{Hard}(U,s)$ keeps
$s$ rows of $U$ with the largest $\ell_2$ row-norm, while setting to
zero other rows.

Suppose that the target statistical parameter $\Theta^*$ is in
$\Xi(r^*, s_1^*,s_2^*)$. The sparsity level $s_1^*$ and $s_2^*$ as
well as the rank $r^*$ are not known in practice, but are needed in
Algorithm~\ref{algo:AltGD}. For the convergence proof we require that
the input parameters to the algorithm are set as $s_1 = c\cdot s_1^*$
and $s_2 = c\cdot s_2^*$ for some $c > 1$.  From simulations, we
observe that the estimation accuracy is not very sensitive to the
choice of $s_1$ and $s_2$ as long as they are chosen greater than the
true values $s_1^*$ and $s_2^*$. This suggests that in practice, we
could set $s_1$ and $s_2$ to be reasonably large values whenever a
reasonable guess of the sparsity level is available, as incorrectly
omitting nonzero value (false negative) is more troublesome than
including one zero value (false positive).  Alternatively, as we do in
simulations, we can use a validation set or an information criteria to
select these tuning parameters. For example, \cite{She2017Selective}
develops the scale-free predictive information criterion to select the
best sparsity parameters. The rank $r$ can be estimated as in
\cite{Bunea2011Optimal}.

To the best of our knowledge, GDT is the first gradient based
algorithm to deal with
a nonconvex optimization problem over a parameter space that is
simultaneously low rank and row and column sparse. In the following
section we will provide conditions on the objective function $f$ and
the starting point $\Theta^0$ which guarantee linear convergence to
$\Theta^*$ up to a statistical error.  As an application, we consider
the multi-task learning problem in Section~\ref{sec:MTL}. We show that
the statistical error nearly matches the optimal minimax rate, while
the algorithm achieves the best performance in terms of estimation and
prediction error in simulations.

\begin{algorithm}[tb]
   \caption{Gradient Descent with Hard Thresholding (GDT)}
   \label{algo:AltGD}
\begin{algorithmic}[1]
   \STATE {\bfseries Input:} Initial estimate $\Theta^0$
   \STATE {{\bfseries Parameters:} Step size $\eta$, Rank $r$, Sparsity level $s_1, s_2$, Total number of iterations $T$}
   \STATE $(\tilde U, \tilde\Sigma, \tilde V) = $ rank $r$ SVD of $\Theta^0$
   \STATE $ U^0 = \text{Hard} (\tilde U(\tilde\Sigma)^{\frac 12},s_1), V^0 = \text{Hard}( \tilde V(\tilde\Sigma)^{\frac 12}, s_2)$
   \FOR{$t=1$ {\bfseries to} $T$}
   \STATE $V^{t+0.5} = V^{t} - \eta\nabla_V f( U^{t},V^{t}) - \eta\nabla_V g( U^{t},V^{t})$, 
   \STATE $V^{t+1} = \text{Hard} (V^{t+0.5},s_2)$
   \STATE $U^{t+0.5} = U^{t} - \eta\nabla_U f( U^{t},V^{t}) - \eta\nabla_U g( U^{t},V^{t})$, 
   \STATE $U^{t+1} = \text{Hard} (U^{t+0.5},s_1)$
   \ENDFOR
   \STATE {\bfseries Output:}  $\Theta^{T} =  U^{T}(V^{T})^{\top}$
\end{algorithmic}
\end{algorithm}

%% file: Theoretical.tex
\section{Theoretical Result}
\label{sec:theoretical}

In this section, we formalize the conditions and state the main
result on the linear convergence of our algorithm. 
We begin in Section~\ref{sec:conditions} by
stating the conditions on the objective function $f$ and
initialization that are needed for our analysis. In
Section~\ref{sec:main-result}, we state Theorem~\ref{main} that 
guarantees linear convergence under the conditions to a statistically 
useful point. The proof outline is given in Section~\ref{sec:proof}.
In Section~\ref{sec:MTL} to follow, we derive results for multi-task 
learning as corollaries of our main result. 

\subsection{Regularity Conditions}
\label{sec:conditions}

We start by stating mild conditions on the objective function $f$,
which have been used in the literature on high-dimensional estimation
and nonconvex optimization, and they hold with high-probability for a
number of statistical models of interest \cite{Zhao2015Nonconvex,
  Zhang2017Nonconvex, Ha2017Alternating}.  Note that all the conditions depend on the choice of $s_1$ and $s_2$ (or equivalently, on $c$). 
  
For $\Theta^* \in \Xi(r^*, s_1^*, s_2^*) $, let
$\Theta^* = U_{\Theta^*}\Sigma_{\Theta^*}V_{\Theta^*}^\top$ be its
singular value decomposition.  Let
$U^* = U_{\Theta^*}\Sigma_{\Theta^*}^{1/2}$ and
$V^* = V_{\Theta^*}\Sigma_{\Theta^*}^{1/2}$ be the balanced
decomposition of $\Theta^* = U^*V^{*\top}$. Note that the
decomposition is not unique as $\Theta^* = (U^*O)(V^*O)^{\top}$ for
any orthogonal matrix $O \in \Ocal(r)$.  Let
$\sigma_1(\Theta^*) = \sigma_{\max}(\Theta^*)$ and
$\sigma_r(\Theta^*) = \sigma_{\min}(\Theta^*)$ denote the maximum and
minimum nonzero singular values of $\Theta^*$. The first condition is
Restricted Strong Convexity and Smoothness on $f$.

{\bf Restricted Strong Convexity and Smoothness (RSC/RSS).}
There exist universal constants $\mu$ and $L$ such that
\begin{equation}
\label{eq:def_strongly_convex}
\frac{\mu }{2} \|\Theta_2 - \Theta_1\|_F^2 \leq f(\Theta_2) - f(\Theta_1) - \langle \nabla f(\Theta_1), \Theta_2 - \Theta_1 \rangle \leq \frac{L }{2} \|\Theta_2 - \Theta_1\|_F^2
\end{equation}
for all $\Theta_1, \Theta_2 \in \Xi(2r, \tilde s_1, \tilde s_2)$ where
$\tilde s_1 = (2c+1)s_1^*$ and $\tilde s_2 = (2c+1)s_2^*$.

\vspace{2mm}
The next condition is on the initial estimate $\Theta^0$. It quantifies how
close the initial estimator needs to be to $\Theta^*$ so that iterates of GDT 
converge to statistically useful solution.

\vspace{2mm}
{\bf Initialization (I).} Define
$\mu_{\min} = \frac 18 \min\{1, \frac{\mu L }{\mu +L }\}$ and 
\begin{equation}
\label{eq:def_I_0}
I_0 = \frac 45 \mu_{\min}\sigma_r(\Theta^*) \cdot \min\Big\{ \frac{1}{\mu +L }, 2 \Big\}.
\end{equation} 
We require
\begin{equation}
\label{eq:ass_init}
\|\Theta^0 - \Theta^*\|_F \leq \frac 15 \min\Big\{ \sigma_r(\Theta^*), \frac{I_0}{\xi} \sqrt{\sigma_r(\Theta^*)} \Big\},
\end{equation} 
where $\xi^2 = 1 + \frac{2}{\sqrt{c-1}}$.

We note that, in general, \eqref{eq:ass_init} defines a ball of
constant radius around $\Theta^*$ in which the initial estimator needs
to fall into. In particular, when considering statistical learning
problems, the initial estimator can be inconsistent as the sample size
increases.

\vspace{2mm}
Next, we define the notion of the statistical error, 
\begin{equation}
\label{eq:def_stat_error}
e_{\text{stat}} = \sup_{\substack{\Delta \in \Xi(2r, \tilde s_1, \tilde s_2) \\ \|\Delta\|_F \leq 1}} \langle \nabla f(\Theta^*), \Delta \rangle.
\end{equation}
Note that the statistical error quantifies how large the gradient of
the objective evaluated at the true parameter $\Theta^*$ can be in the
directions of simultaneously low-rank and sparse matrices.  It
implicitly depends on the choice of $c$ and as we will see later there
is a trade-off in balancing the statistical error and convergence rate
of GDT. As $c$ increases, statistical error gets larger, but requires
us to choose a smaller step size in order to guarantee convergence.

\vspace{2mm}
With these two conditions, we are ready to the choice of the step size in
Algorithm~\ref{algo:AltGD}.

\vspace{3mm}
{\bf Step Size Selection.} We choose the step size $\eta$ to satisfy
\begin{equation}
\label{eq:eta_constant}
\eta \leq \frac{1}{16\|Z_0\|_2^2} \cdot \min\Big\{\frac{1}{2(\mu +L )}, 1\Big\}, 
\end{equation}
Furthermore, we require $\eta$ and $c$ to satisfy
\begin{equation}
\begin{aligned}
\label{eq:def_beta}
\beta = \xi^2 \rbr{1 - \eta\cdot
\frac{2}{5} \mu_{\min}\sigma_r(\Theta^*)
} < 1,
\end{aligned}
\end{equation}
and 
\begin{equation}
\begin{aligned}
\label{eq:ass_estat}
e_{\rm stat}^2 \leq \frac{1-\beta}{\xi^2\eta} \cdot \frac{L \mu }{L  + \mu }\cdot  I_0^2.
\end{aligned}
\end{equation}

The condition that the step size $\eta$ satisfies
\eqref{eq:eta_constant} is typical in the literature on convex
optimization of strongly convex and smooth functions. Under
\eqref{eq:def_beta} we will be able to show contraction after one
iteration and progress towards $\Theta^*$.  The second term in
\eqref{eq:def_beta} is always smaller than $1$, while the first term
$\xi^2$ is slightly larger than $1$ and is the price we pay for the
hard thresholding step. In order to show linear convergence we need to
balance the choice of $\eta$ and $\xi^2$ to ensure that $\beta < 1$.
From \eqref{eq:def_beta}, we see that if we select a small step size
$\eta$, then we need to have a small $\xi^2$, which means a large $c$.
Intuitively, if $\eta$ is too small, it may be impossible to change
row and column support in each iteration. In this case we have to keep
many active rows and columns to make sure we do not miss the true
signal. This leads to large $s_1$ and $s_2$, or equivalently to a
large $c$. However, the statistical error \eqref{eq:def_stat_error}
will increase with increase of $c$ and these are the trade-off on the
selection of $\eta$ and $c$.

Finally, \eqref{eq:ass_estat} guarantees that the iterates do not run
outside of the initial ball given in \eqref{eq:ass_init}.  In case
\eqref{eq:ass_estat} is violated, then the initialization point
$\Theta^0$ is already a good enough estimate of $\Theta^*$. Therefore,
this requirement is not restrictive. In practice, we found that the
selection of $\eta$ and $c$ is not restrictive and the convergence is
guaranteed for a wide range of values of their values.

\subsection{Main Result}
\label{sec:main-result}

Our main result establishes linear convergence of GDT iterates to
$\Theta^*$ up to statistical error.  Since the factorization of
$\Theta^*$ is not unique, we turn to measure the subspace distance of
the iterates $({U}^t, {V}^t)$ to the balanced decomposition
$ U^*(V^*)^\top = \Theta^*$.

{\bf Subspace distance.} Let
$Z^* = \sbr{\begin{array}{c}U^* \\ V^* \end{array}}$ where
$\Theta^* = U^*{V^*}^\top$ and $\sigma_i(U^*) = \sigma_i(V^*)$ for
each $i = 1, ..., r$. Define the subspace distance between
$Z = \sbr{\begin{array}{c}U \\ V \end{array}}$ and
$Z^* = \sbr{\begin{array}{c}U^* \\ V^* \end{array}}$ as
\begin{equation}
d^2(Z,Z^*) = \min_{O \in \Ocal(r)} \cbr{\|U - U^*O \|_F^2 + \|V - V^*O \|_F^2}.
\end{equation}

With this, we are ready to state our main result.

\begin{theorem}
\label{main}
Suppose the conditions {\bf (RSC/RSS)}, {\bf (I)} are satisfied and
the step size $\eta$ satisfies \eqref{eq:eta_constant} -
\eqref{eq:ass_estat}.
Then after $T$ iterations of GDT (Algorithm
  \ref{algo:AltGD}), we have
\begin{equation}
\begin{aligned}
d^2(Z^{T}, Z^*) \leq \beta^T \cdot d^2(Z^0, Z^*)
 +
\frac{\xi^2\eta}{1-\beta}\cdot\frac{L  + \mu }{L \cdot\mu }\cdot e_{\rm stat}^2 .
\label{theoremUV}
\end{aligned}
\end{equation}
Furthermore, for $\Theta^{T} =  U^{T}(V^{T})^{\top}$ we have
\begin{equation}
\begin{aligned}
\|\Theta^{T} - \Theta^*\|_F^2 \leq 4\sigma_1(\Theta^*) \cdot \Big[ \beta^T \cdot d^2(Z^0, Z^*) + \frac{\xi^2\eta}{1-\beta}\cdot\frac{L  + \mu }{L \cdot\mu }\cdot e_{\rm stat}^2 \Big].
\label{theoremtheta}
\end{aligned}
\end{equation}
\end{theorem}

The proof sketch of Theorem~\ref{main} is given in the following
section.  Conceptually, Theorem~\ref{main} provides a minimal set of
conditions for convergence of GDT. The first term in equations
\eqref{theoremUV} and \eqref{theoremtheta} correspond to the
optimization error, whereas the second term corresponds to the
statistical error. These bounds show that the distance between the
iterates and $\Theta^*$ drop exponentially up to the
statistical limit $e_{\text{stat}}$, which is problem specific. In
statistical learning problem, it commonly depends on the sample size
and the signal-to-noise ratio of the problem.

Theorem~\ref{main} provides convergence in a statistical setting to
the ``true'' parameter $\Theta^*$. However, as mentioned in
Section~\ref{sec:methodology}, Algorithm~\ref{algo:AltGD} and
Theorem~\ref{main} can also be used to establish linear convergence to
a global minimizer in a deterministic setting. Suppose
$(\hat U, \hat V) \in \arg\min_{U \in \Ucal, V \in \Vcal}\{f(U,V)\}$
is a global minimizer and $\hat \Theta = \hat U \hat V^{\top}$.
Furthermore, assume that the conditions in Section
\ref{sec:conditions} are satisfied with $\hat \Theta$ in place of
$\Theta^*$. Then we have that the iterates $\{\Theta^t\}$ obtained by
GDT converge linearly to a global minimum $\hat \Theta$ up to the
error $\hat e_{\text{stat}}$ defined similar to
\eqref{eq:def_stat_error} with $\hat \Theta$ in place of
$\Theta^*$. This error comes from sparsity and hard thresholding.  In
particular, suppose there are no row or column sparsity constraints in
the optimization problem \eqref{eq:opt:rep}, so that we do not have
hard-thresholding steps in Algorithm~\ref{algo:AltGD}. Then we have
$\hat e_{\text{stat}} = 0$, so that iterates $\{\Theta^t\}$ converge
linearly to $\hat \Theta$, recovering the result of
\cite{Zhao2015Nonconvex}.

\subsection{Proof Sketch of Theorem \ref{main}}
\label{sec:proof}

In this section we sketch the proof of our main result.  The proof
combines three lemmas. We first one quantify the accuracy of the
initialization step. The following one quantifies the improvement in
the accuracy by one step of GDT. The third lemma shows that the step
size assumed in Theorem~\ref{main} satisfies conditions of the second
lemma.  Detailed proofs of these lemmas are relegated to
Section~\ref{sec:appendix}.

\vspace{2mm}
Our first lemma quantifies the accuracy of the initialization step.
\begin{lemma} 
\label{lemmaInit}
  Suppose that the input to GDT, $\Theta^0$, satisfies
  initialization condition \eqref{eq:ass_init}.  Then the initial
  iterates $ U^0$ and $V^0$ obtained in lines $3$ and $4$ of
  Algorithm~\ref{algo:AltGD} satisfy
\begin{equation}
\label{eq:bound_I_0}
d(Z^0, Z^*) \leq I_0,
\end{equation}
where $Z^0 = \sbr{\begin{array}{c}U^0 \\ V^0 \end{array}}$ and $I_0$
is defined in \eqref{eq:def_I_0}.
\end{lemma}

The proof of Lemma~\ref{lemmaInit} is given in Section~\ref{A:lemmaInit}.

\begin{lemma}
\label{lemma:iteration_contraction}
Suppose the conditions {\bf (RSC/RSS)}, {\bf (I)} are
satisfied. Assume that the point
$Z = \sbr{\begin{array}{c}U \\ V \end{array}}$ satisfies
$d(Z, Z^*) \leq I_0$.
Let $(U^+, V^+)$ denote the next iterate obtained with Algorithm~\ref{algo:AltGD} with 
the step size $\eta$ satisfying 
\begin{equation}
\label{eq:step_size_condition}
\eta \leq \frac{1}{8\|Z\|_2^2} \cdot \min\Big\{\frac{1}{2(\mu +L )}, 1\Big\}.
\end{equation}
Then we have
\begin{equation}
\label{eq:lemma_contraction_result}
d^2(Z^+, Z^*) \leq \xi^2 \bigg[ \Big(1 - \eta\cdot
\frac{2}{5} \mu_{\min}\sigma_r(\Theta^*)
\Big)\cdot d^2(Z, Z^*)
 +
\eta\cdot\frac{L  + \mu }{L \cdot\mu }\cdot e_{\rm stat}^2 \bigg],
\end{equation}
where $\xi^2 = 1 + \frac{2}{\sqrt{c - 1}}$.
  
\end{lemma}

The proof of Lemma~\ref{lemma:iteration_contraction} is given in Section~\ref{A:lemma:iteration_contraction}.

\begin{lemma}
\label{lemma:step_size_constant}
Suppose $Z = \sbr{\begin{array}{c}U \\ V \end{array}}$ satisfies
$d(Z, Z^*) \leq I_0$. We have that the choice of step size
\eqref{eq:eta_constant} in Theorem \ref{main} satisfies the condition
\eqref{eq:step_size_condition} in Lemma
\ref{lemma:iteration_contraction}.
\end{lemma}

The proof of Lemma~\ref{lemma:step_size_constant} is given in
Section~\ref{A:lemma:step_size_constant}.

\vspace{4mm}
Combining the three results above, we can complete the proof of
Theorem~\ref{main}.  Starting from initialization $\Theta^0$
satisfying the initialization condition \eqref{eq:ass_init},
Lemma~\ref{lemmaInit} ensures that \eqref{eq:bound_I_0} is satisfied
for $Z^0$ and Lemma \ref{lemma:step_size_constant} ensures that the
choice of step size \eqref{eq:eta_constant} satisfies the step size
condition \eqref{eq:step_size_condition} in Lemma
\ref{lemma:iteration_contraction}. We can then apply Lemma 3 and get
the next iterate $Z^1 = Z^+$, which satisfies
\eqref{eq:lemma_contraction_result}. Using the condition on
statistical error \eqref{eq:ass_estat}, initialization
\eqref{eq:ass_init}, and a simple calculation, we can verify that
$Z^1$ satisfies $d(Z^1, Z^*) \leq I_0$. Therefore we can apply Lemma
\ref{lemmaInit}, Lemma \ref{lemma:iteration_contraction}, and Lemma
\ref{lemma:step_size_constant} repeatedly to obtain
\begin{equation}
d^2(Z^{t+1}, Z^*) \leq \beta \cdot d^2(Z^t, Z^*)
 +
\xi^2\eta\cdot\frac{L  + \mu }{L \cdot\mu }\cdot e_{\rm stat}^2, 
\end{equation}
for each $t = 0, 1, ..., T$. We then have
\begin{equation}
d^2(Z^{T}, Z^*) \leq \beta^T \cdot d^2(Z^0, Z^*)
 +
\frac{\xi^2\eta}{1-\beta}\cdot\frac{L  + \mu }{L \cdot\mu }\cdot e_{\rm stat}^2 .
\end{equation}
Finally, for $\Theta^{T} =  U^{T}(V^{T})^{\top}$, let $O^T \in \Ocal(r)$ be such that
\[
d^2(Z^T, Z^*) = \|U^T - U^* O^T \|_F^2 + \|V^T - V^* O^T \|_F^2.
\]
We have
\begin{equation}
\begin{aligned}
\|\Theta^{T} - \Theta^*\|_F^2 &= \|  U^{T}(V^{T})^{\top} -  U^{*}O^T (V^{*}O^T)^{\top} \|^2_F\\
&\leq \Big[ \| U^{T}\|_2\|V^{T}-V^{*}O^T\|_F + \|V^{*}\|_2\|  U^{T} -  U^{*}O^T \|_F \Big]^2 \\
&\leq \| U^{T}\|_2^2\|V^{T}-V^{*}O^T\|_F^2 + \|V^{*}\|_2^2\|  U^{T} -  U^{*}O^T \|_F^2 \\
&\leq 2\|Z^*\|_2^2 \cdot d^2(Z^T, Z^*) \\
&\leq 4\sigma_1(\Theta^*) \cdot \Big[ \beta^T \cdot d^2(Z^0, Z^*) + \frac{\xi^2\eta}{1-\beta}\cdot\frac{L  + \mu }{L \cdot\mu }\cdot e_{\rm stat}^2 \Big],
\end{aligned}
\end{equation}
which shows linear convergence up to the statistical error.


%% file: Statistical_error.tex
\section{Application to Multi-task Learning}
\label{sec:MTL}

In this section, we apply the theory developed in
Section~\ref{sec:theoretical} on two specific problems.  First, in
Section~\ref{sec:GDT-multi-task}, we apply GDT algorithm to a
multi-task learning problem. We show that under commonly used
statistical conditions the conditions on the objective function
stated in Section~\ref{sec:conditions} are satisfied with
high-probability. Next, in Section~\ref{subsec:MTRL} we discuss an
application to multi-task reinforcement learning problem.

\subsection{GDT for Multi-task Learning}
\label{sec:GDT-multi-task}
 
We apply GDT algorithm to the problem of multi-task learning, which
has been successfully applied in a wide range of application areas,
ranging from neuroscience \citep{Vounou2012Sparse}, natural language
understanding \citep{collobert2011natural}, speech recognition
\citep{seltzer2013multi}, computer vision \citep{She2017Selective},
and genetics \citep{yu2017multitask} to remote sensing
\citep{xue2007multi}, image classification \citep{lapin2014scalable},
spam filtering \citep{weinberger2009feature}, web search
\citep{chapelle2010multi}, disease prediction
\citep{zhou2013modeling}, and eQTL mapping \citep{kim2010tree}.  By
transferring information between related tasks it is hoped that
samples will be better utilized, leading to improved generalization
performance.

We consider the following linear multi-task learning problem
\begin{equation}
\label{eqn:MTL_model}
Y = X \Theta^* + E,
\end{equation}
where $Y \in \RR^{n\times k}$ is the response matrix,
$X \in \RR^{n \times p}$ is the matrix of predictors,
$\Theta^* \in \RR^{p\times k}$ is an unknown matrix of coefficients,
and $E \in \RR^{n \times k}$ is an unobserved noise matrix with
i.i.d.~mean zero and variance $\sigma^2$ entries.  Here $n$ denotes
the sample size, $k$ is the number of responses, and $p$ is the number
of predictors.

There are a number of ways to capture relationships between different
tasks and success of different methods relies on this relationship.
\cite{Evgeniou2004Regularized} studied a setting where linear
predictors are close to each other. In a high-dimensional setting,
with large number of variables, it is common to assume that there are
a few variables predictive of all tasks, while others are not
predictive \cite{turlach2005simultaneous, Obozinski10Support,
  Lounici2011Oracle, kolar11union, Wang2015Distributed}.  Another
popular condition is to assume that the predictors lie in a shared
lower dimensional subspace \cite{Ando2005framework,
  Amit2007Uncovering, Yuan2007Dimension, Argyriou08Convex,
  Wang2016Distributed}.  In contemporary applications, however, it is
increasingly common that both the number of predictors and the number
of tasks is large compared to the sample size. For example, in a study
of regulatory relationships between genome-wide measurements, where
micro-RNA measurements are used to explain the gene expression levels,
it is commonly assumed that a small number of micro-RNAs regulate
genes participating in few regulatory pathways
\cite{Ma2014Learning}. In such a setting, it is reasonable to assume
that the coefficients are both sparse and low rank.  That is, one
believes that the predictors can be combined into fewer latent
features that drive the variation in the multiple response variables
and are composed only of relevant predictor variables. Compared to a
setting where either variables are selected or latent features are
learned, there is much less work on simultaneous variable selection
and rank reduction \cite{Bunea2011Optimal, Chen2011Reduced,
  Chen2012Sparse, She2017Selective}.  In addition, we when both $p$
and $k$ are large, it is also needed to assume the column sparsity on
the matrix $\Theta^*$ to make estimation feasible
\cite{Ma2014Adaptive}, a model that has been referred to as the
two-way sparse reduced-rank regression model. We focus on this model
here.

{\bf Multi-task Model (MTM)} In the model~\eqref{eqn:MTL_model}, we
assume that the true coefficient matrix
$\Theta^* \in \Xi(r, s_1^*, s_2^*)$. The noise matrix $E$ has
i.i.d. sub-Gaussian elements with variance proxy $\sigma^2$, which
requires that each element $e_{ij}$ satisfies $\mathbb E(e_{ij}) = 0$
and its moment generating function satisfies
$\mathbb E[\exp(t e_{ij})] \leq \exp(\sigma^2t^2/2)$.  The design
matrix $X$ is considered fixed with columns normalized to have mean
$0$ and standard deviation $1$. Moreover, we assume $X$ satisfies the following
Restricted Eigenvalue (RE) condition \cite{negahban2010unified} for some constant $\underline \kappa(s_1)$ and $\bar \kappa(s_1)$. 
\begin{equation}
\underline \kappa(s_1) \cdot \|\theta\|_2^2 \leq \frac{1}{n} \|X\theta\|_2^2 \leq \bar \kappa(s_1) \cdot \|\theta\|_2^2 \,\,\, \text{ for all } \, \|\theta\|_0 \leq s_1.
\end{equation}


We will show that under the condition {\bf (MTM)}, GDT converges
linearly to the optimal coefficient $\Theta^*$ up to a region of
statistical error.  Compared to the previous methods for estimating
jointly sparse and low rank coefficients \cite{Bunea2011Optimal,
  Chen2011Reduced, Chen2012Sparse, Ma2014Adaptive}, GDT is more
scalable and improves estimation accuracy as illustrated in the
simulation Section~\ref{sec:experiment}.

In the context of the multi-task learning with the model in
\eqref{eqn:MTL_model}, we are going to use the least squares loss. The
objective function in is
$f(\Theta) = \frac{1}{2n} \| Y - X\Theta \|_F^2$ and we write
$\Theta = UV^\top$ with $U \in \R^{p \times r}$ and
$V \in \R^{k \times r}$. The constraint set is set as before as
$U \in \Ucal(s_1)$ and $V \in \Ucal(s_2)$ with
$s_1 = c\cdot s_1^*, s_2 = c\cdot s_2^*$ for some $c>1$.  The rank $r$
and the sparsity levels $s_1,s_2$ are tuning parameters, which can be
selected using the information criterion as in
\cite{She2017Selective}.

In order to apply the results of Theorem~\ref{main}, we first verify
the conditions in Section~\ref{sec:conditions}. The condition {\bf
(RSC/RSS)} in is equivalent to 
\begin{equation}
\mu \big\|\Theta_2 - \Theta_1\big\|_F^2 \leq \Big\langle \frac 1n X^\top X(\Theta_2 - \Theta_1), \Theta_2 - \Theta_1 \Big\rangle \leq L \big\|\Theta_2 - \Theta_1\big\|_F^2, 
\end{equation}
and it holds with
$\mu = \underline \kappa(s_1)$
and
$L = \bar \kappa(s_1)$.

Next, we discuss how to initialize GDT in the context of multi-task
learning. Under the structural conditions on $\Theta^*$ in the
condition {\bf (MTM)} there are a number of way to obtain an initial
estimator ${\Theta}^0$. For example, we can use row and column
screening \cite{fan08sis}, group lasso \cite{Yuan2006Model}, and lasso
\cite{tibshirani96regression} among other procedures. Here and in
simulations we use the lasso estimator, which takes the form
\[
{\Theta}^0 = \arg\min_{\Theta \in \RR^{p \times k}}
\frac{1}{2n} \| Y - X\Theta \|_F^2 + \lambda \|\Theta\|_1.
\]
The benefit of this approach is that it is scalable to the
high-dimensional setting and trivially parallelizable, since each
column of ${\Theta}^0$ can be estimated separately.  The requirement
of the initialization condition {\bf (I)} is effectively a requirement
on the sample size. Under the condition {\bf (MTM)}, a result of
\cite{negahban2010unified} shows that these conditions are satisfied
with $n \geq s_1^*s_2^*\log p\log k$.

\vspace{2mm}
We then characterize the statistical error $e_{\text{stat}}$ under the
condition {\bf (MTM)}.

\begin{lemma}
\label{lemma:e_stat}
Under the condition {\bf (MTM)}, with probability at least $1 - (p \vee k)^{-1}$ we have
\begin{equation}
e_{\text{stat}} \leq C\sigma \sqrt{\frac{(s_1^*+s_2^*)\big(r + \log(p \vee k)\big)}{n}}
\end{equation}
for some constant $C$.
 
\end{lemma}

The proof of Lemma \ref{lemma:e_stat} is given in Section \ref{A:lemma:e_stat}.

\vspace{2mm}
With these conditions, we have the following result on GDT when
applied to the multi-task learning model in \eqref{eqn:MTL_model}.

\begin{corollary}
Suppose that the condition {\bf (MTM)} is satisfied. 
Then for all 
\begin{equation}
T \geq C \log\bigg[\frac {n}{(s_1^*+s_2^*) \big(r+\log (p\vee k)\big)}\bigg], 
\end{equation}
with probability at least $1 - (p \vee k)^{-1}$, we have
\begin{equation}
\begin{aligned}
\label{eq:stat_error}
\|\Theta^{T} - \Theta^*\|_F \leq C \sigma \sqrt{\frac{(s_1^*+s_2^*) \big(r+\log (p\vee k)\big)}{n}}
\end{aligned}
\end{equation}
for some constant $C$.
\label{stat_error_thm}
\end{corollary}

Each iteration of the algorithm requires computing the gradient step with time complexity $r(n+r)(p+k)$. 
Note that if there is no error term $E$ in the model
\eqref{eqn:MTL_model}, then Algorithm \ref{algo:AltGD} converges
linearly to the true coefficient matrix $\Theta^*$, since
$e_{\text{stat}} = 0$ in that case.  The error rate in
Corollary~\ref{stat_error_thm} matches the error rate of the algorithm
proposed in \cite{Ma2014Adaptive}. However, our algorithm does not
require a new independent sample in each iteration and allows for
non-Gaussian errors.  Compared to the minimax rate
\begin{equation}
\label{eq:minimiax_rate}
\sigma \sqrt{\frac 1n \bigg[(s_1^*+s_2^*)r + s_1^*\log\frac{ep}{s_1^*} + s_2^*\log\frac{ek}{s_2^*} \bigg]}
\end{equation}
established in \cite{Ma2014Adaptive}, both our algorithm and that of
\cite{Ma2014Adaptive} match the rate up to a multiplicative log
factor.  To the best of our knowledge, achieving the minimax rate
\eqref{eq:minimiax_rate} with a computationally scalable procedure is
still an open problem.  Note, however, that when $r$ is comparable to
$\log(p\vee k)$ the rates match up to a constant multiplier. Therefore for
large enough $T$, GDT algorithm attains near optimal rate.

In case we do not consider column sparsity, that is, when $s_2^* = k$,
Corollary~\ref{stat_error_thm} gives error rate
\begin{equation}
\label{eq:stat_error_row}
\|\Theta^{T} - \Theta^*\|_F \leq C\sigma \sqrt{\frac{kr + s_1^*\big(r+\log p\big)}{n}}
\end{equation}
and prediction error
\begin{equation}
\label{eq:prediction_GDT}
\|X\Theta^{T} - X\Theta^*\|_F^2 \leq C\sigma^2\Big(kr + s_1^*\big(r+\log p\big)\Big).
\end{equation}
Compared to the prediction error bound $kr + s_1^* r \log{\frac ps}$
proved in \cite{Bunea2012Joint}, we see that GDT error is much smaller
with $r+\log p$ in place of $r\log p$. Moreover, GDT error matches the
prediction error $(k+s_1^*-r)r + s_1^*\log{p}$ established in
\cite{She2017Selective}, as long as $k \geq Cr$ which is typically
satisfied.


%% file: MTRL.tex
\subsection{Application to Multi-task Reinforcement Learning}
\label{subsec:MTRL}

Reinforcement learning (RL) and approximate dynamic programming (ADP)
are popular algorithms that help decision makers find optimal policies
for decision making problems under uncertainty that can be cast in the
framework of Markov Decision Processes (MDP) \cite{bertsekas1995neuro,
  sutton1998introduction}. Similar to many other approaches, when the
sample size is small these algorithms may have poor performance. A
possible workaround then is to simultaneously solve multiple related
tasks and take advantage of their similarity and shared
structure. This approach is called multi-task reinforcement learning
(MTRL) and has been studied extensively \cite{Lazaric2010Bayesian,
  wilson2007multi, Snel2012Multi}. In this section we show how GDT
algorithm can be applied to the MTRL problem.

A Markov decision process (MDP) is represented by a 5-tuple
$\mathcal M = (S, A, P, R, \gamma)$ where $S$ represents the state
space (which we assume to be finite for simplicity); $A$ is a finite
set of actions; $P_a(s,s') = \Pr(s_{t+1}=s' \mid s_t = s, a_t=a)$ is
the Markovian transition kernel that measures the probability that
action $a$ in state $s$ at time $t$ will lead to state $s'$ at time
$t+1$ (we assume $P_a$ to be time homogeneous); $R(s,a)$ is the
state-action reward function measuring the instantaneous reward
received when taking action $a$ in state $s$; and $\gamma$ is the
discount factor.  The core problem of MDP is to find a deterministic
policy $\pi: S \to A$ that specifies the action to take when decision
maker is in some state $s$.  Define the Bellman operator
\begin{equation}
\mathcal T Q(s,a) = R(s,a) + \gamma \sum_{s'}P_a(s,s')\max_{a'} Q(s',a'),
\end{equation}
where $Q: S \times A \to \mathbb R$ is the state-action value function. The MDP can then be solved by calculating the optimal state-action value function $Q^*$ which gives the total discounted reward obtained starting in state $s$ and taking action $a$, and then following the optimal policy in subsequent time steps. Given $Q^*$, the optimal policy is recovered by the greedy policy: $\pi^*(s) = \arg\max_{a \in A}Q^*(s,a)$. 

In MTRL the objective is to solve $k$ related tasks simultaneously
where each task $k_0 \in \{1, \ldots, k \}$ corresponds to an MDP:
$\mathcal M_{k_0} = (S, A, P_{k_0}, R_{k_0}, \gamma_{k_0})$. Thus,
these $k$ tasks share the same state and action space but each task
has a different transition dynamics $P_{k_0} $, state-action reward
function $R_{k_0} $, and discount factor $\gamma_{k_0} $. The decision
maker's goal is to find an optimal policy for each MDP. If these MDPs
do not share any information or structure, then it is straightforward
to solve each of them separately. Here we assume the MDPs do share
some structure so that the $k$ tasks can be learned together with
smaller sample complexity than learning them separately.

We follow the structure in \cite{calandriello2014sparse} and solve
this MTRL problem by the fitted-$Q$ iteration (F$Q$I) algorithm
\cite{ernst2005tree}, one of the most popular method for ADP. In
contrast to exact value iteration ($Q^{t} = \mathcal T Q^{t-1}$), in
F$Q$I this iteration is approximated by solving a regression problem
by representing $Q(s,a)$ as a linear function in some features
representing the state-action pairs. To be more specific, we denote
$\varphi(s) = [\varphi_1(s), \varphi_2(s), ..., \varphi_{p_s}(s)]$ as
the feature mapping for state $s$ where $\varphi_i: S \to \mathbb R$
denotes the $i$th feature. We then extend the state-feature vector
$\varphi$ to a feature vector mapping state $s$ and action $a$ as:
\begin{equation}
\phi(s,a) = [\underbrace{0, 0, ..., 0}_{(a-1)\times p_s \text{ times}}, \varphi_1(s), \varphi_2(s), ..., \varphi_{p_s}(s), \underbrace{0, 0, ..., 0}_{(|A| - a)\times p_s \text{ times}}] \in \mathbb R^p,
\end{equation}
where $p = |A| \times p_s$. 
Finally, for MDP $k_0$, we represent the state-action value function $Q_{k_0}(\cdot,\cdot)$ as an $|S|\times |A|$ dimensional column vector with:
\[ Q_{k_0}(s,a) = \phi(s,a)^{\top} \cdot \Theta_{k_0} \]
where $\Theta_{k_0}$ is a $p \times 1$ dimensional column vector. If
$\Theta \in \mathbb{R}^{p \times k}$ represents the matrix with
columns $\Theta_{k_0}$, $k \in \{1,\ldots, k\}$, then we see that
given the $Q_{k_0}(s,a)$ state-action value functions, estimating the
$\Theta$ matrix is just a Multi-Task Learning problem of the form
\eqref{eqn:MTL_model} with the response matrix
$Y \doteq Q \in \mathbb{R}^{n\times k}$ where $n=|S|\times |A|$
denotes the ``sample size'' with rows indexed by pairs
$(s,a) \in S \times A$, $X \doteq \Phi \in \mathbb{R}^{n \times p}$
represents the matrix of predictors (features) with $(s,a)^{th}$ row
as $\phi(s,a)$, and $\Theta^*$ is the unknown matrix of ADP
coefficients.  Consistent with the GDT algorithm, to exploit shared
sparsity and structure across the $k$ MDP tasks, we will subsequently
assume that the coefficient matrix $\Theta^*$ is row sparse and low
rank.

Algorithm~\ref{algo:MTRL} provides details of MTRL with GDT.  We
assume we have access to the generative model of the $k$ MDPs so that
we can sample reward $r$ and state $s'$ from $R(s,a)$ and
$P_a(s,s')$. With ``design states''
$S_{k} \subseteq S, \ n_s \doteq |S_k|$ given as input, for each
action $a$ and each state $s \in S_{k}$, F$Q$I first generates samples
(reward $r$ and transition state $s'$) from the generative model of
each MDP. These samples form a new dataset according to
\begin{equation}
y_{i,a,k_0}^t = r^t_{i,a,k_0} + \gamma \max_{a'} \hat Q^{t-1}_{k_0}({s'}_{i,a,k_0}^t, a').
\end{equation}
Here $\hat Q^{t-1}_{k_0}$ is calculated using the coefficient matrix from previous iteration:
\begin{equation}
\hat Q^{t-1}_{k_0}({s'}_{i,a,k_0}^t, a') = \phi({s'}_{i,a,k_0}^{t}, a')^{\top}\cdot \Theta^{t-1}_{k_0} 
\end{equation}
We then build dataset
$\mathcal D^t_{k_0} = \big\{(s_{i},a), y_{i,a,k_0}^t \big\}_{s_i\in
  S_k, a \in A }$
with $s$ as predictor and $y$ as response, and apply GDT algorithm on
the dataset $\{D^t_{k_0}\}_{k_0=1}^k$ to get estimator
$\Theta^{t}$. This completes an iteration $t$ and we repeat this
process until convergence. Finally the optimal policy $\pi_{k_0}^t$ is
given by greedy policy:
$\pi_{k_0}^t(s) = \arg\max_{a \in A} \hat{Q}_{k_0}^t(s,a)$ at each
iteration $t$.

To derive theoretical result analogous to
\cite{calandriello2014sparse}, we further assume $R(s,a) \in [0,1]$
and hence the maximum cumulative discounted reward
$Q_{\max} = 1/(1-\gamma)$. Since each task is a meaningful MDP, we do
not assume sparsity on columns.  Suppose
$\sup_s\| \varphi(s)\|_2 \leq L$, we have the following theoretical
result:

\begin{theorem}
\label{thm:MTRL}
Suppose the linear model holds and suppose the conditions in Section \ref{sec:theoretical} are satisfied for each $\Theta^*_a$ with rank $r$ and row sparsity $s_1^*$, then after $T$ iterations, with probability at least $\big(1 - (p \wedge k)^{-1}\big)^{T}$ we have
\begin{equation}
\begin{aligned}
\label{eq_thm:MTRL}
\frac 1k \sum_{k_0=1}^k \Big\| Q_{k_0}^* - Q_{k_0}^{\pi_{k_0}^{T}} \Big\|_2^2 &\leq \frac{C}{(1-\gamma)^4}\bigg[ \frac 1nQ_{\max}^2L^4 \Big(r+\frac{s_1^*}{k} (r+\log p) \Big) \bigg] 
+ \frac{4Q_{\max}^2}{(1-\gamma)^4} \bigg[ C \beta^{T} + \gamma^{T} \bigg]^2
\end{aligned}
\end{equation}
for some constant $C$.
\end{theorem}

\begin{proof}
We start from the intermediate result in \cite{munos2008finite}:
\begin{equation}
\Big|Q^*_{k_0} - Q_{k_0}^{\pi_{k_0}^{T}} \Big| \leq \frac{2\gamma(1-\gamma^{T+1})}{(1-\gamma)^2} \bigg[ \sum_{t = 0}^{T-1} \alpha_t|\epsilon_{k_0}^t| + \alpha_{T}\big|Q_t^* - Q_t^0\big| \bigg],
\end{equation}
where
\begin{equation}
\alpha_t = \frac{(1-\gamma)\gamma^{T-t-1}}{1-\gamma^{T+1}}\text{, for $t < T$, and } \alpha_{T} =  \frac{(1-\gamma)\gamma^{T}}{1-\gamma^{T+1}}.
\end{equation}

The error term $\epsilon_{k_0}^t(s', b)$ measures the approximation error in state $s' \in S$ and action $b \in A$. It can be bounded by
\begin{equation}
\big|\epsilon_{k_0}^t(s', b)\big| = \big|\varphi(s')^{\top}\Theta^t_{k_0,b} - \varphi(s')^{\top}\Theta^*_{k_0,b}\big| \leq \big\|\varphi(s')\big\|_2\big\|\Theta^t_{k_0,b} - \Theta^*_{k_0,b} \big\|_2 \leq L \big\|\Theta^t_{k_0,b} - \Theta^*_{k_0,b} \big\|_2.
\end{equation}

We then have
\begin{equation}
\Big|Q^*_{k_0} - Q_{k_0}^{\pi_{k_0}^{T}} \Big| \leq \frac{2\gamma(1-\gamma^{T+1})}{(1-\gamma)^2} \bigg[ \sum_{t = 0}^{T-1} \alpha_t L \max_b \big\|\Theta^t_{k_0,b} - \Theta^*_{k_0,b} \big\|_2 + 2\alpha_{T}Q_{\max} \bigg].
\end{equation}

Taking average, and plugging in the main result \eqref{theoremtheta} and the statistical error \eqref{eq:stat_error_row} we obtain our desired result.
\end{proof}

\begin{algorithm}[tb]
   \caption{Multi-Task Reinforcement Learning with GDT}
   \label{algo:MTRL}
\begin{algorithmic}
   \STATE {\bfseries Input: States $S_k = \{{s_i}\}_{i=1}^{n_s} \subseteq S$.}
   \STATE {\bfseries Initialize $\Theta^0 = 0$}
   \FOR{$t=1$ {\bfseries to} $T$}
   \FOR{$a=1$ {\bfseries to} $|A|$}
   \FOR{$k_0=1$ {\bfseries to} $k$, $i=1$ {\bfseries to} $n_s$}
   \STATE Generate samples $r^t_{i,a,k_0} = R_{k_0}(s_{i},a)$ and ${s'}_{i,a,k_0}^{t} \sim P_{a, k_0}( s_{i},s')$
   \STATE Calculate $y_{i,a,k_0}^t = r^t_{i,a,k_0} + \gamma \max_{a'} \hat Q^{t-1}_{k_0}({s'}_{i,a,k_0}^t, a')$
   \ENDFOR
   \ENDFOR
   \STATE Estimate $\Theta^t$ using GDT algorithm with $X = \left\{ X((s_i,a),\cdot) = \phi(s_i,a)^\top \right\}_{s_i \in S_k , a \in A}$ and $Y = \left\{ Y((s_i,a),k_0) = y^{t}_{i,a,k_0} \right\}_{s_i \in S, a \in A, k_0 \in [k]}$.
   \ENDFOR
   \STATE {\bfseries Output:}  $\Theta^{T} $
\end{algorithmic}
\end{algorithm}


%% file: Experiment.tex
\section{Experiment}
\label{sec:experiment}

In this section we demonstrate the effectiveness of the GDT algorithm by
extensive experiments\footnote{The codes are available at \url{http://home.uchicago.edu/~ming93/research.html}}. Section \ref{sec:simulation} shows results on
synthetic datasets while Section \ref{sec:real1} and \ref{sec:real2}
show results on two real datasets.

\subsection{Synthetic Datasets}
\label{sec:simulation}

We present numerical experiments on MTL problem to support our
theoretical analysis. Throughout this section, we generate the
instances by sampling all entries of design matrix $X$, all nonzero
entries of the true signal $U^*$ and $V^*$, and all entries of the
noise matrix $E$ as i.i.d. standard normal.

\vspace{2mm}
{\bf Linear convergence.} We first demonstrate our linear convergence
result. Because it is hard to quantify linear convergence with
statistical error, we turn to show the linear convergence in some
special cases.  Firstly, as we discussed after
Corollary~\ref{stat_error_thm}, suppose there is no error term $E$ in
the model \eqref{eqn:MTL_model}, then Algorithm \ref{algo:AltGD}
converges linearly to the true coefficient matrix $\Theta^*$.  In this
case we choose $p = 100, k = 50, r = 8, s_1^* = s_2^* = 10$, and the
estimation error is shown in Figure \ref{linear_no_error}.  Secondly,
as we discussed at the end of Section \ref{sec:main-result}, suppose
there are no row or column sparsity constraints on $\Theta^*$, then
Algorithm \ref{algo:AltGD} converges linearly to global minimum
$\hat \Theta$.  In this case it is more likely that we are in low
dimensions, therefore we choose $p = 50$. The estimation error is
shown in Figure \ref{linear_low}. We see that in both cases GDT has
linear convergence rate.

\begin{centering}
\begin{figure*}[tbp]
\begin{minipage}[t]{0.5\linewidth}
\centering
\includegraphics[width=0.9\textwidth]{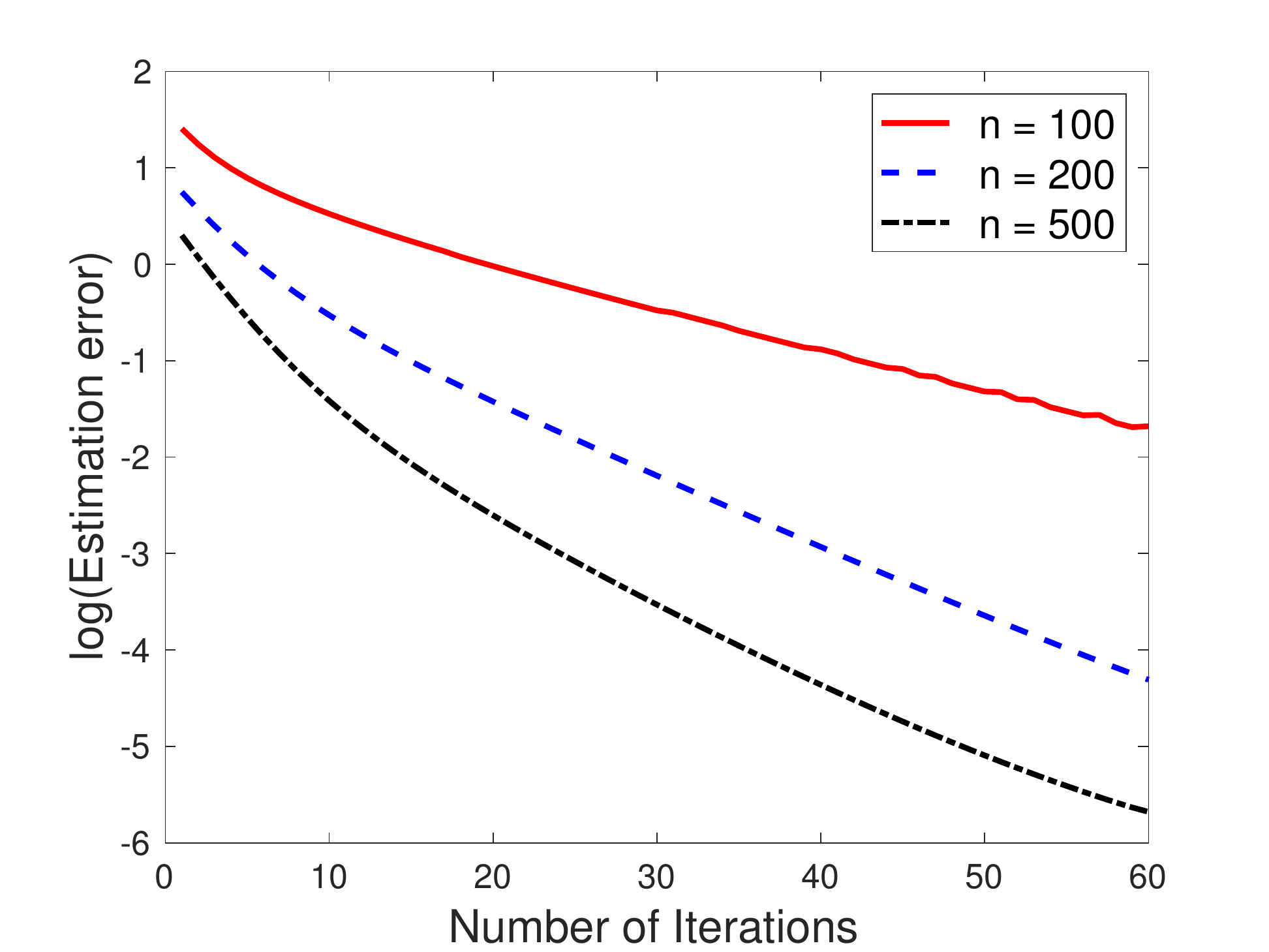}
\caption{No error case}
\label{linear_no_error}
\end{minipage}
\begin{minipage}[t]{0.5\linewidth}
\centering
\includegraphics[width=0.9\textwidth]{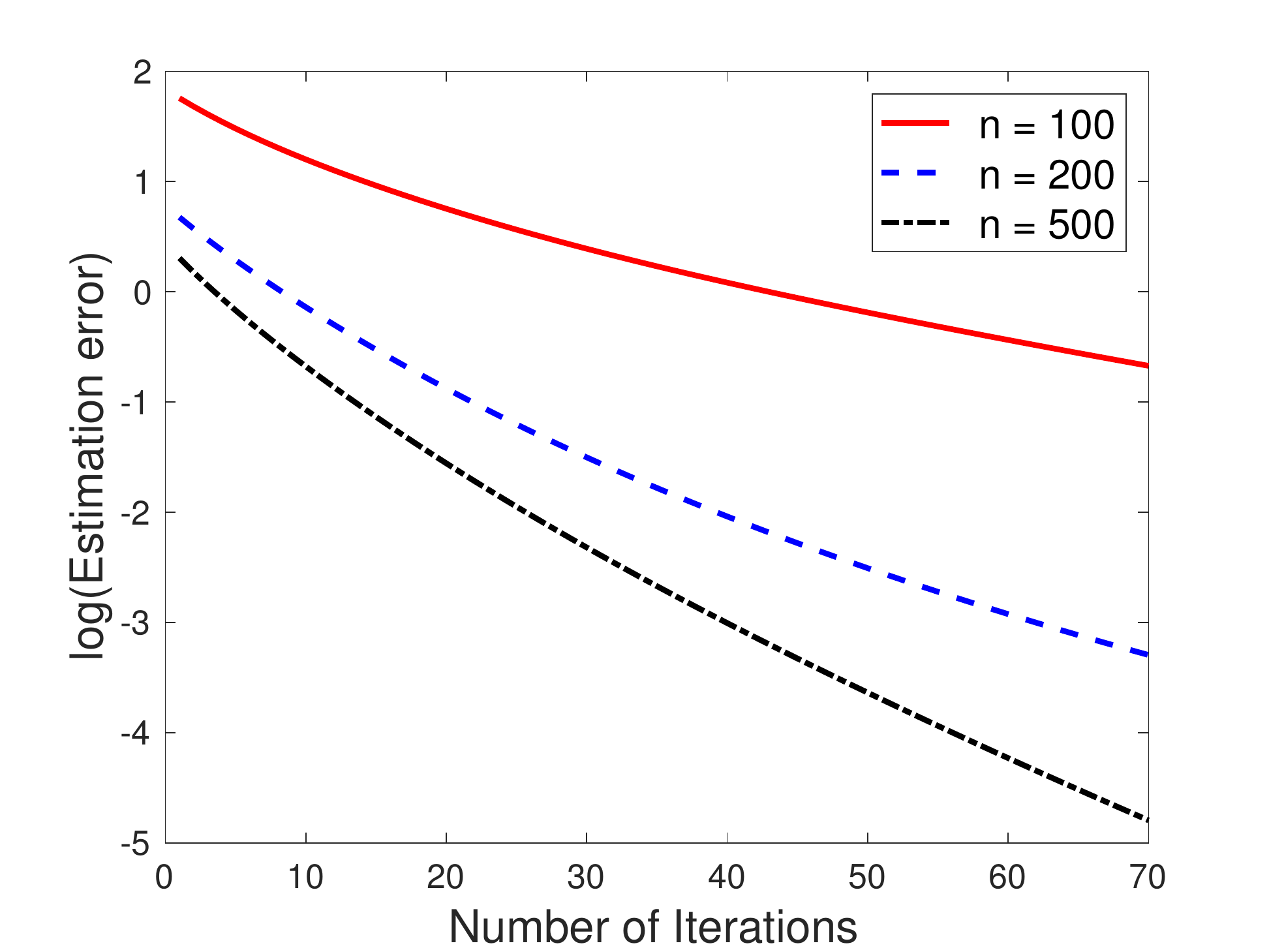}
\caption{No sparsity case}
\label{linear_low}
\end{minipage}
\end{figure*}
\end{centering}

\vspace{2mm}
{\bf Estimation accuracy.}  We compare our algorithm with the
Double Projected Penalization (DPP) method in \cite{Ma2014Adaptive},
the thresholding SVD method (TSVD) method in \cite{Ma2014Learning},
the exclusive extraction algorithm (EEA) in \cite{Chen2011Reduced},
the two methods (denoted by RCGL and JRRS) in \cite{Bunea2012Joint},
and the standard Multitask learning method (MTL, with $L_{2,1}$
penalty).  Here we set $n=50,p=100,k=50,r=8,s_1^* = s_2^* = 10$. The
reason why we choose a relatively small scale is that many other
methods do not scale to high dimensions, as will shown in Table
\ref{runningtime}. We will show the effectiveness of our method in
high dimensions later.  Except for standard MTL, all the other methods
need an estimate of the rank to proceed for which we apply the rank
estimator in \cite{Bunea2011Optimal}. For the methods that rely on 
tuning parameters, we generate an independent validation set
to select the tuning parameters. 

We consider two coefficient matrix settings, one is only row sparse
and the other one is both row sparse and column sparse.  We also
consider strong signal and weak signal settings. The strong signal
setting is described above and for the weak signal setting, we divide
the true $\Theta^*$ by 5, resulting in a signal for which recovering
true non-zero variables becomes much more difficult.  Table \ref{row} (strong
signal, row sparse), Table \ref{rowcolumn} (strong signal, row
\emph{and} column sparse), Table \ref{row_weak} (weak signal, row
sparse) and Table \ref{rowcolumn_weak} (weak signal, row \emph{and}
column sparse) report the mean and the standard deviation of
prediction errors, estimation errors and size of selected models based
on 50 replications in each setting.  We can see that in all the cases
GDT has the lowest estimation error and prediction error.  When the
signal is weak, GDT may underselect the number of nonzero
rows/columns, but it still has the best performance.

\vspace{2mm} {\bf Running time.}  We then compare the running time of
all these methods. We fix a baseline model size
$n=50,p=80,k=50,r=4,s_1^* = s_2^* = 10$, and set a free parameter
$\zeta$. For $\zeta = \{1,5,10,20,50,100\}$, each time we increase
$n,p,s_1^*,s_2^*$ by a factor of $\zeta$ and increase $k,r$ by a
factor of $\lfloor\sqrt\zeta\rfloor$ and record the running time (in
seconds) of each method for a fixed tolerance level, whenever
possible. If for some $\zeta$ the algorithm does not converge in 2
hours then we simply record ``$>$2h'' and no longer increase $\zeta$
for that method.  Table \ref{runningtime} summarizes the results. We
can see that GDT is fast even in very high dimension, while all of the
other methods are computationally expensive.  We note that even though
GDT uses the lasso estimator in the initialization step, all the
variables are used in the subsequent iterations and not only the ones
selected by the lasso. In particular, the speed of the method does
not come from the initialization step.

\begin{table}[tp]
\caption{Strong signal, Row sparse}
\begin{center}
{
\begin{tabular}{cccc}
\hline
& Estimation error & Prediction error & $|$Row support$|$ \\ \hline
{\bf GDT} & 0.0452  $\pm$  0.0110   & 1.1060  $\pm$  0.0248 &  10.16  $\pm$    0.51        \\
DPP &  0.0584  $\pm$  0.0113   & 1.1290  $\pm$  0.0357  & 52.64  $\pm$  15.2 \\
TSVD & 0.3169 $\pm$  0.1351 & 2.4158 $\pm$  0.9899 & 25.62 $\pm$  8.03 \\
EEA &  0.3053 $\pm$  0.0998 & 1.2349  $\pm$ 0.0362 & 84.28  $\pm$ 6.70 \\
RCGL & 0.0591  $\pm$  0.0148 &   1.1101  $\pm$  0.0168  & 49.60  $\pm$ 10.6  \\
JRRS & 0.0877 $\pm$   0.0227  &  1.1857  $\pm$  0.0214  & 12.26  $\pm$  2.02 \\
MTL & 0.0904  $\pm$  0.0243  &  1.1753  $\pm$  0.0204  & 73.40  $\pm$  2.67 \\
\hline
\end{tabular}
}
\end{center}
\label{row}
\end{table}%

\begin{table}[tp]
\caption{Strong signal, Row sparse and column sparse}
\begin{center}
{
\begin{tabular}{ccccc}
\hline
& Estimation error & Prediction error & $|$Row support$|$ & $|$Column support$|$ \\ \hline
{\bf GDT} & 0.0624  $\pm$  0.0121  &  1.0353  $\pm$  0.0167 &  10.24    $\pm$  0.65  & 10.24  $\pm$  0.68\\
DPP &  0.0921  $\pm$ 0.0251  &  1.0790  $\pm$  0.0295 &  54.10  $\pm$ 18.25  & 10.38  $\pm$  0.60\\
TSVD & 0.3354 $\pm$  0.1053 & 1.7600 $\pm$  0.3415 & 28.66 $\pm$ 7.27 & 30.88 $\pm$  8.46  \\
EEA & 0.2604 $\pm$ 0.1159 & 1.1023 $\pm$ 0.0220  & 64.44 $\pm$ 9.88 & 12.10 $\pm$ 2.69 \\
RCGL & 0.1217  $\pm$  0.0325  &  1.1075  $\pm$  0.0174  & 42.06  $\pm$  7.93  & 50   $\pm$      0 \\
JRRS & 0.1682  $\pm$  0.0410  &  1.1612  $\pm$  0.0174  & 13.96  $\pm$  4.69  & 50    $\pm$     0 \\
MTL& 0.1837  $\pm$  0.0499 &   1.1652   $\pm$ 0.0160 &  73.50   $\pm$ 3.17  & 50     $\pm$    0\\\hline
\end{tabular}
}
\end{center}
\label{rowcolumn}
\end{table}%


\begin{table}[tp]
\caption{Weak signal, Row sparse}
\begin{center}
{
\begin{tabular}{cccc}
\hline
& Estimation error & Prediction error & $|$Row support$|$ \\ \hline

{\bf GDT} & 0.2328 $\pm$ 0.0474 & 1.1282 $\pm$ 0.0231 & 10.08 $\pm$ 0.56\\ 
DPP &0.2954 $\pm$ 0.0640 & 1.1624 $\pm$ 0.0315 & 47.26 $\pm$ 11.7\\ 
TSVD &0.5842 $\pm$ 0.1020 & 1.4271 $\pm$ 0.0903 & 30.81 $\pm$ 4.72\\ 
EEA & 0.3802 $\pm$ 0.0787 & 1.1647 $\pm$ 0.0206 & 46.16 $\pm$ 8.97\\
RCGL & 0.2775 $\pm$ 0.0605 & 1.1493 $\pm$ 0.0291 & 37.92 $\pm$ 14.4\\ 
JRRS & 0.3600 $\pm$ 0.0752 & 1.1975 $\pm$ 0.0392 & 11.74 $\pm$ 1.35 \\
MTL& 0.3577 $\pm$ 0.0721 & 1.2140 $\pm$ 0.0418 & 69.92 $\pm$ 12.8\\ 
\hline
\end{tabular}
}
\end{center}
\label{row_weak}
\end{table}%


\begin{table}[tp]
\caption{Weak signal, Row sparse and column sparse}
\begin{center}
{
\begin{tabular}{ccccc}
\hline
& Estimation error & Prediction error & $|$Row support$|$ & $|$Column support$|$ \\ \hline

{\bf GDT} & 0.3173 $\pm$ 0.0949 & 1.0380 $\pm$ 0.0218 & 9.56 $\pm$ 1.56 & 10.06 $\pm$ 1.21\\ 
DPP &0.3899 $\pm$ 0.0737 & 1.0580 $\pm$ 0.0216 & 50.66 $\pm$ 12.86 & 13.52 $\pm$ 5.02\\ 
TSVD &0.6310 $\pm$ 0.1074 & 1.1372 $\pm$ 0.0246 & 49.94 $\pm$ 5.53 & 43.38 $\pm$ 2.55\\ 
EEA & 0.6016 $\pm$ 0.0965 & 1.0874 $\pm$ 0.0197 & 30.64 $\pm$ 8.65 & 30.64 $\pm$ 8.65\\
RCGL & 0.4601 $\pm$ 0.0819 & 1.1017 $\pm$ 0.0262 & 28.9 $\pm$ 12.36 & 50 $\pm$ 0\\ 
JRRS & 0.5535 $\pm$ 0.0866 & 1.1164 $\pm$ 0.0262 & 12.42 $\pm$ 6.02 & 50 $\pm$ 0 \\
MTL& 0.5776 $\pm$ 0.0873 & 1.1286 $\pm$ 0.0296 & 53.0 $\pm$ 18.41 & 50 $\pm$ 0\\ \hline
\end{tabular}
}
\end{center}
\label{rowcolumn_weak}
\end{table}%

\begin{table}[tp]
\caption{Running time comparison (in seconds)}
\begin{center}
{
\begin{tabular}{ccccccc}
\hline
& $\zeta = 1$ & $\zeta = 5$ & $\zeta = 10$ & $\zeta = 20$ & $\zeta = 50$ & $\zeta = 100$   \\ \hline
{\bf GDT} &  0.11 & 0.20 & 0.51 & 2.14 & 29.3 & 235.8\\
DPP & 0.19 & 0.61 & 3.18 & 17.22 & 315.4 & 2489\\
TSVD & 0.07 & 1.09 & 6.32 & 37.8 & 543 & 6075 \\
EEA & 0.50 & 35.6 & 256 & $>$2h & $>$2h & $>$2h \\
RCGL & 0.18& 1.02 & 7.15 & 36.4 & 657.4 & $>$2h\\
JRRS & 0.19 & 0.82 & 6.36 & 30.0 & 610.2 & $>$2h\\
MTL& 0.18 & 3.12 & 30.92 & 184.3  & $>$2h & $>$2h\\\hline
\end{tabular}
}
\end{center}
\label{runningtime}
\end{table}%

\vspace{2mm}
{\bf Effectiveness in high dimension.}  We finally demonstrate the
effectiveness of GDT algorithm in high dimensions. Table \ref{row} and
Table \ref{rowcolumn} are both in low dimensions because we want to
compare with other algorithms and they are slow in high dimensions, as
shown in Table \ref{runningtime}.  Now we run our algorithm only and
we choose $p = 5000, k = 3000, r = 50, s_1^* = s_2^* = 100$. The
estimation error and objective value are shown in Figure
\ref{high_estimation} and Figure \ref{high_obj}, respectively. In each
figure, iteration 0 is for initialization we obtained by
Lasso. 
We can see that both estimation error and objective value continue to
decrease, which demonstrates the effectiveness and necessity of GDT
algorithm.  From Figure \ref{high_estimation} we also find that early
stopping can help to avoid overfitting (although not too much),
especially when $n$ is small.

\begin{centering}
\begin{figure*}[tbp]
\begin{minipage}[t]{0.5\linewidth}
\centering
\includegraphics[width=0.9\textwidth]{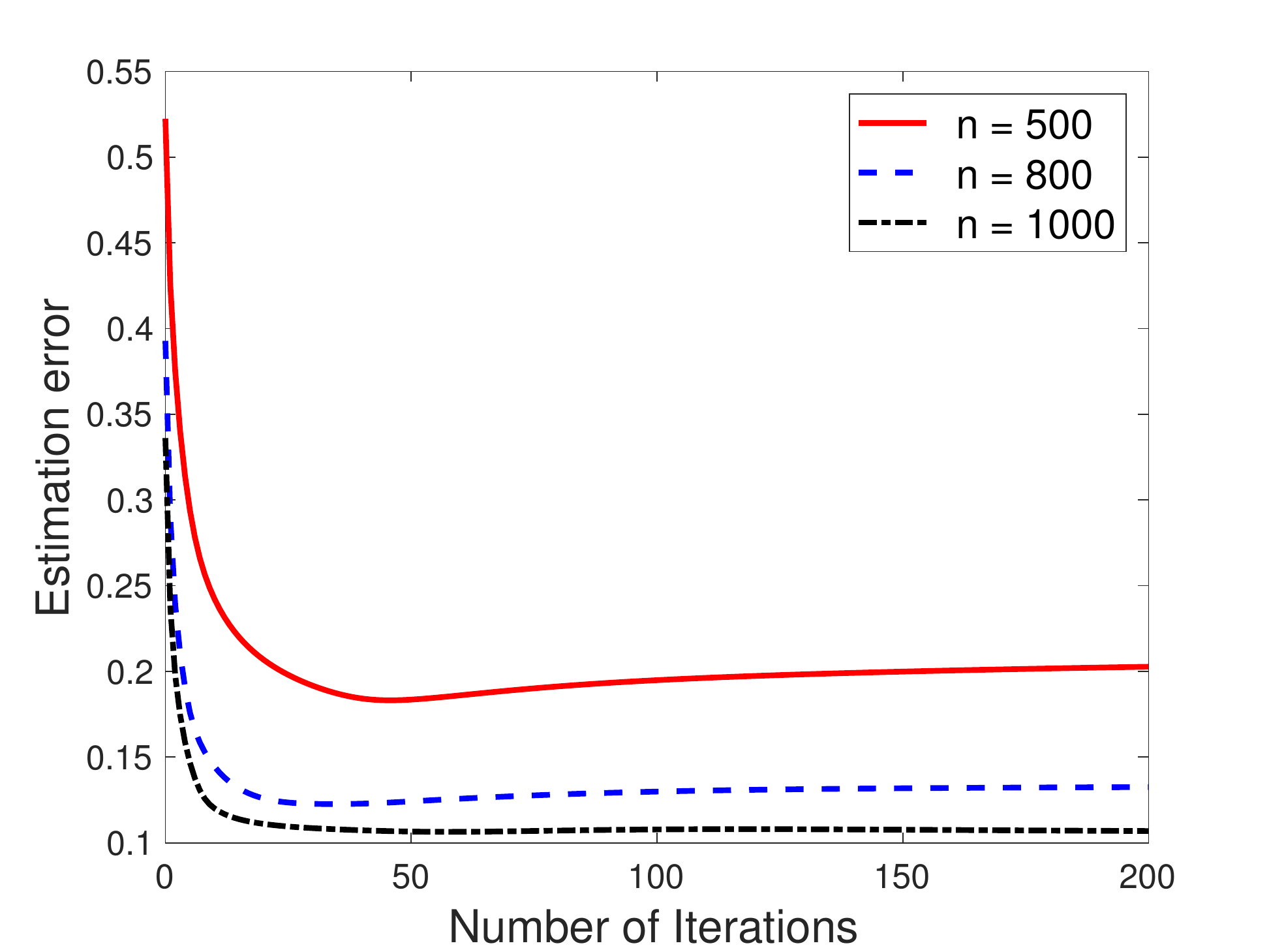}
\caption{Estimation error}
\label{high_estimation}
\end{minipage}
\begin{minipage}[t]{0.5\linewidth}
\centering
\includegraphics[width=0.9\textwidth]{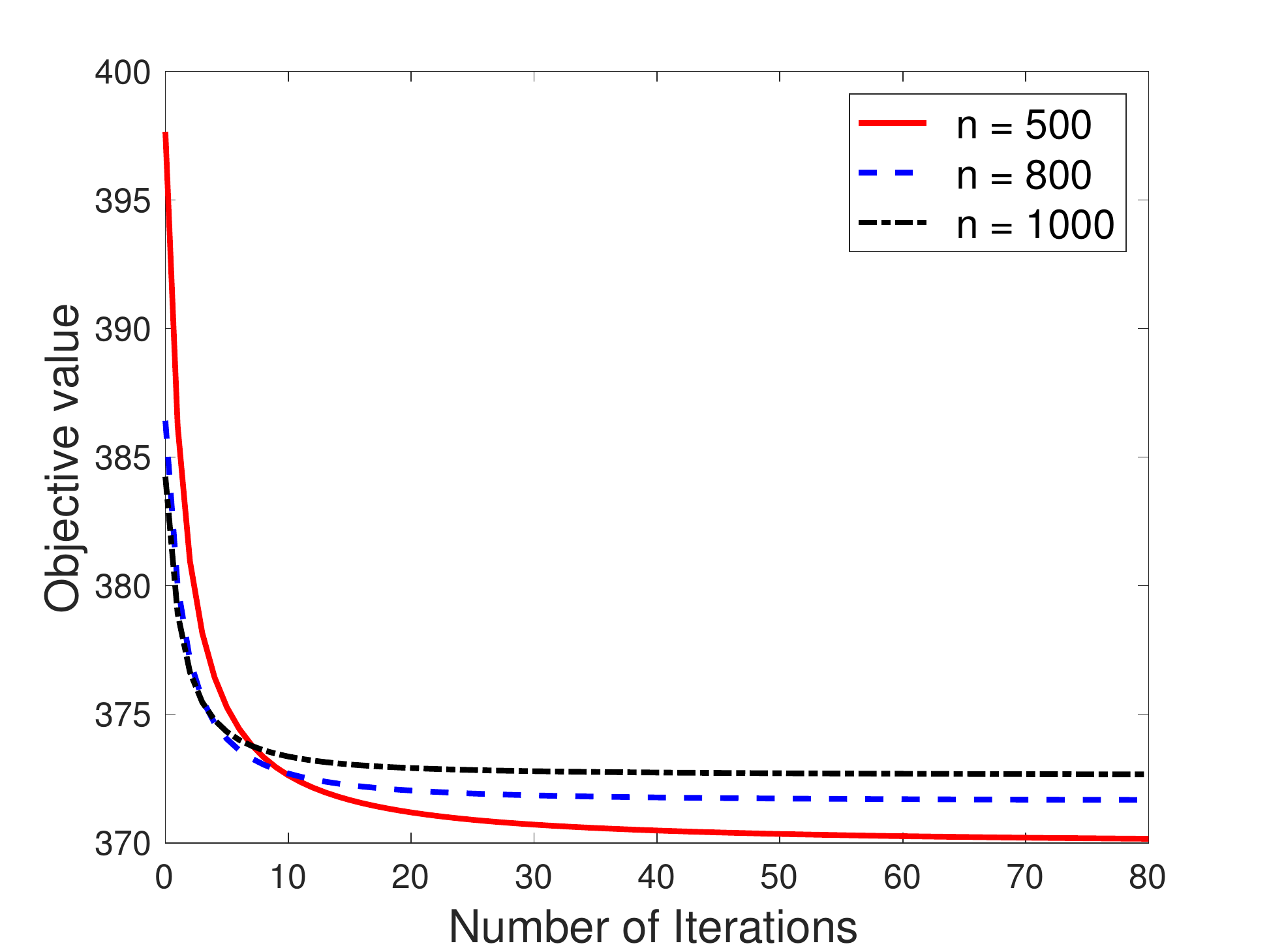}
\caption{Objective value}
\label{high_obj}
\end{minipage}
\end{figure*}
\end{centering}

\subsection{Norwegian Paper Quality Dataset}
\label{sec:real1}

In this section we apply GDT to Norwegian paper quality dataset.  This
data was obtained from a controlled experiment that was carried out at
a paper factory in Norway to uncover the effect of three control
variables $X_1,X_2,X_3$ on the quality of the paper which was measured
by 13 response variables. Each of the control variables $X_i$ takes
values in $\{-1,0,1\}$. To account for possible interactions and
nonlinear effects, second order terms were added to the set of
predictors, yielding
$X_1, X_2, X_3, X_1^2, X_2^2, X_3^2, X_1 \cdot X_2, X_1 \cdot X_3,
X_2\cdot X_3$.

The data set can be downloaded from the website of
\cite{izenman2008modern} and its structure clearly indicates that
dimension reduction is possible, making it a typical application for
reduced rank regression methods \cite{izenman2008modern,
  aldrin1996moderate, Bunea2012Joint, She2017Selective}. Based on the
analysis of \cite{Bunea2011Optimal} and \cite{aldrin1996moderate} we
select the rank $\hat r = 3$; also suggested by
\cite{Bunea2011Optimal} we take $s_1=6$ and $s_2 = k = 13$ which means
we have row sparsity only. GDT selects 6 of the original 9 predictors,
with $X_1^2,X_1\cdot X_2$ and $X_2\cdot X_3$ discarded, which is
consistent with the result in \cite{Bunea2011Optimal}.

To compare prediction errors, we split the whole dataset at random,
with 70\% for training and 30\% for test, and repeat the process 50
times to compare the performance of the above methods. All tuning
parameters are selected by cross validation and we always center the
responses in the training data (and transform the test data
accordingly). The average RMSE on test set is shown in Table
\ref{RMSE}.  We can see that GDT is competitive with the best method,
demonstrating its effectiveness on real datasets.

\begin{table}[htp]
\caption{RMSE on paper quality dataset}
\begin{center}
{\footnotesize
\begin{tabular}{ccccccc}
\hline
{\bf GDT} & DPP & TSVD & EEA & RCGL & JRRS  & MTL \\\hline
1.002 & 1.012 & 1.094 & 1.161 & 1.001 & 1.013 & 1.014 \\\hline
\end{tabular}
}
\end{center}
\label{RMSE}
\end{table}%

\subsection{Calcium Imaging Data}
\label{sec:real2}

As a microscopy technique in neuroscience, calcium imaging is gaining
more and more attentions \cite{haeffele2014structured}. It records
fluorescent images from neurons 
and allows us to identify the spiking activity of the neurons.
To achieve this goal, \cite{pnevmatikakis2014structured} introduces a
spatiotemporal model and we briefly introduce this model here. 
More detailed description can be found in \cite{pnevmatikakis2014structured} and \cite{Ma2014Adaptive}. 
Denote $k = \ell_1 \times \ell_2$ as the pixels we observe, and denote $K$ as the total number of neurons. The observation time step is $t = 1, ..., T$.
Let 
$S \in \mathbb{R}^{T \times K}$ be the number of spikes at each time step and for each neuron;
$A \in \mathbb{R}^{K \times k}$ be the nonnegative spatial footprint for each neuron at each pixel;
$Y \in \mathbb{R}^{T \times k}$ be the observation at each time step and at each pixel; and 
$E \in \mathbb{R}^{T \times k}$ be the observation error. 
Ignore the baseline vector for all the pixels, the model in \cite{pnevmatikakis2014structured} is given by 
\begin{equation}
\begin{aligned}
Y &= G^{-1}SA + E = X \Theta^* + E
\end{aligned}
\label{calcium_MTL}
\end{equation}
where $\Theta^* = SA$ is the coefficient
matrix and $X = G^{-1}$ is observed with
\begin{equation}
G = 
\begin{pmatrix}
    1 & 0 & \dots & 0 \\
    -\gamma & 1 & \ddots  & \vdots \\
    \vdots & \ddots & \ddots & 0 \\
    0 & \dots & -\gamma & 1
\end{pmatrix}.
\label{calcium_G}
\end{equation}

Here $\gamma$ is set to be $\gamma = 1 - 1/(\text{frame rate})$ as suggested by
\cite{vogelstein2010fast}. 
From the settings we see that each row of $S$ is the activation for all the neurons, and therefore it is natural to have $S$ to be row sparse since usually we would not observe too many activations in a fixed time period;
also, each column of $A$ is the footprint for all the neurons at each pixel, and therefore it is natural to have $A$ to be column sparse since we expect to see only a few neurons in a fixed area.
Therefore our coefficient matrix
$\Theta^* = SA$ would be both row sparse and column sparse. 
It is also low
rank since it is the product of two ``tall" matrices because the
number of neurons $K$ are usually small. Now we see this is a multi-task
learning problem with simultaneous row-sparse, column-sparse and low
rank coefficient matrix where $n = p = T$ and
$k = \ell_1 \times \ell_2$.

We consider the calcium imaging data in
\cite{akerboom2012optimization} 
which is a movie with 559 frames (acquired at approximately 8.64
frames/sec), where each frame is $135 \times 131$ pixels.  
This dataset is also analyzed in \cite{Ma2014Adaptive} and \cite{haeffele2014structured}. 
For this dataset, we have $n=p=559$ and $k=135 \times 131 = 17,685$. We use $r = 50$,
more conservative than the estimator given by \cite{Bunea2011Optimal}
and we set $s_1 = 100$ row sparsity and $s_2 = 3000$ column sparsity.
Figure \ref{true_signals} shows five most significant manually labeled
regions; Figure \ref{our_signals} are the corresponding signals
estimated by our GDT algorithm. We can see that they match very
well, which demonstrates the effectiveness of our method.

\begin{figure}[t]
\begin{minipage}[t]{0.192\linewidth}
\centering
\includegraphics[width=\textwidth]{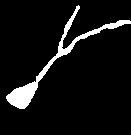}
\end{minipage}
\begin{minipage}[t]{0.192\linewidth}
\centering
\includegraphics[width=\textwidth]{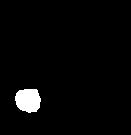}
\end{minipage}
\begin{minipage}[t]{0.192\linewidth}
\centering
\includegraphics[width=\textwidth]{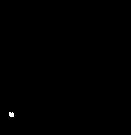}
\end{minipage}
\begin{minipage}[t]{0.192\linewidth}
\centering
\includegraphics[width=\textwidth]{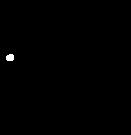}
\end{minipage}
\begin{minipage}[t]{0.192\linewidth}
\centering
\includegraphics[width=\textwidth]{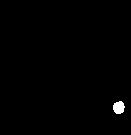}
\end{minipage}
\caption{Manually selected top 5 labeled regions}
\label{true_signals}

\vspace{5mm}

\begin{minipage}[t]{0.192\linewidth}
\centering
\includegraphics[width=\textwidth]{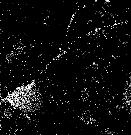}
\end{minipage}
\begin{minipage}[t]{0.192\linewidth}
\centering
\includegraphics[width=\textwidth]{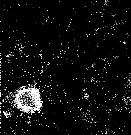}
\end{minipage}
\begin{minipage}[t]{0.192\linewidth}
\centering
\includegraphics[width=\textwidth]{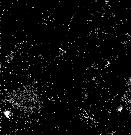}
\end{minipage}
\begin{minipage}[t]{0.192\linewidth}
\centering
\includegraphics[width=\textwidth]{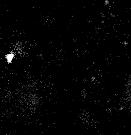}
\end{minipage}
\begin{minipage}[t]{0.192\linewidth}
\centering
\includegraphics[width=\textwidth]{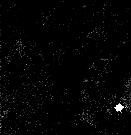}
\end{minipage}
\caption{Corresponding signals estimated by our GDT algorithm}
\label{our_signals}
\end{figure}


%% file: Appendix.tex
\section{Technical Proofs}
\label{sec:appendix}

This section collects technical proofs.

\subsection{Proof of Lemma \ref{lemmaInit}}
\label{A:lemmaInit}
 
Let $[\tilde U, \tilde\Sigma, \tilde V] = \text{rSVD}(\Theta^0)$
be the rank $r$ SVD of the matrix ${\Theta^0}$ and let
\[ \tilde\Theta = \tilde U\tilde\Sigma(\tilde V)^{\top}
 = \arg\min_{{\rm rank}(\Theta) \leq r} \|\Theta -
\Theta^0\|_F.
\]
Since $\tilde\Theta$ is the best rank $r$ approximation to ${\Theta^0}$, we have
\begin{equation}
\|\tilde\Theta - \Theta^0\|_F \leq  \|\Theta^0 - \Theta^*\|_F. 
\label{Init3}
\end{equation}
The triangle inequality gives us 
\begin{equation}
\|\tilde\Theta - \Theta^*\|_F \leq  \|\Theta^0 - \Theta^*\|_F +  \|\Theta^0 - \tilde\Theta\|_F \leq 2 \|\Theta^0 - \Theta^*\|_F.
\label{Init4}
\end{equation}
Now that both $\tilde \Theta$ and $\Theta^*$ are rank $r$ matrices, and
according to \eqref{eq:ass_init} we have
\begin{equation}
\|\tilde\Theta - \Theta^*\|_F \leq 2 \|\Theta^0 - \Theta^*\|_F \leq \frac 12 \sigma_r(\Theta^*).
\label{Init5}
\end{equation}
Then, Lemma 5.14 in \cite{Tu2016Low} gives us 
\begin{equation}
\begin{aligned}
d^2\bigg( \sbr{\begin{array}{c} \tilde U \tilde \Sigma^{\frac 12} \\ \tilde V \tilde \Sigma^{\frac 12} \end{array}}, \sbr{\begin{array}{c}U^* \\ V^* \end{array}} \bigg) &\leq \frac{2}{\sqrt 2-1} \cdot \frac{\|\tilde\Theta - \Theta^*\|_F^2}{\sigma_r(\Theta^*)} \\
&\leq \frac{2}{\sqrt 2-1} \cdot \frac{4}{\sigma_r(\Theta^*)} \cdot \frac{I_0^2}{25\xi^2} \cdot \sigma_r(\Theta^*)\\
&\leq \frac{I_0^2}{\xi^2}
\end{aligned}
\end{equation}
where the second inequality comes from the initialization condition
\eqref{eq:ass_init}. Finally, Lemma 3.3 in \cite{Li2016Stochastic}
gives
\begin{equation}
\begin{aligned}
d^2\bigg( \sbr{\begin{array}{c} U^0\\ V^0 \end{array}}, \sbr{\begin{array}{c}U^* \\ V^* \end{array}} \bigg) &\leq \xi^2
d^2\bigg( \sbr{\begin{array}{c} \tilde U \tilde \Sigma^{\frac 12} \\ \tilde V \tilde \Sigma^{\frac 12} \end{array}}, \sbr{\begin{array}{c}U^* \\ V^* \end{array}} \bigg) \leq I_0^2.
\end{aligned}
\end{equation}

\subsection{Proof of Lemma \ref{lemma:iteration_contraction}}
\label{A:lemma:iteration_contraction}

For notation simplicity, let
$Z = \sbr{\begin{array}{c}U \\ V \end{array}}$ denote the current
iterate and let $Z^+ = \sbr{\begin{array}{c}U^+ \\ V^+ \end{array}}$
denote the next iterate.  Let
$S_U = \Scal(U) \cup \Scal(U^+) \cup \Scal(U^*)$ and
$S_V = \Scal(V) \cup \Scal(V^+) \cup \Scal(V^*)$.  With some abuse of
notation, we define the index set $S_Z = S_U \cup S_V$ to represent
coordinates of $Z$ corresponding to $U_{S_U}$ and $V_{S_V}$.  For an
index set $S$, let
$\Pcal(U, S) = \sbr{\begin{array}{c}U_S \\ 0_{S^C}\end{array}}$. Let
$G(U,V) = f(U,V)+g(U,V)$.  Finally, let $\Delta_U = U - U^*\hat O$,
$\Delta_V = V - V^*\hat O$ and $\Delta_Z = Z - Z^*\hat O$. With these notations, we can write
\[
U^+ = {\rm Hard}(U - \eta \cdot \nabla G_U(U, V), s_1) = {\rm Hard}\rbr{U - \eta\cdot \Pcal\rbr{\nabla G_U(U, V),S_U}, s_1}
\]
and
\[
V^+ = {\rm Hard}(V - \eta \cdot \nabla G_V(U, V), s_2) = {\rm Hard}\rbr{V - \eta \cdot \Pcal\rbr{\nabla G_V(U, V), S_V}, s_2}.
\]
Let $\hat O \in \Ocal(r)$ be such that
\[
d^2(Z, Z^*) = \|U - U^*\hat O \|_F^2 + \|V - V^*\hat O \|_F^2.
\]
We have that 
\begin{equation}
\begin{aligned}
   d^2(Z^+, Z^*) 
&= \min_{O \in \Ocal(r)} \bigg\|
\sbr{
\begin{array}{c}
U^+ \\ V^+
\end{array}
}
-
\sbr{
\begin{array}{c}
U^*{O} \\ V^*{O}
\end{array}
}\bigg\|_F^2 \\
&\leq
\bigg\|
\sbr{
\begin{array}{c}
{\rm Hard}\rbr{U - \eta\cdot \Pcal\rbr{\nabla G_U(U, V), S_U}, s_1} \\ 
{\rm Hard}\rbr{V - \eta\cdot \Pcal\rbr{\nabla G_V(U, V), S_V}, s_2}
\end{array}
}
-
\sbr{
\begin{array}{c}
U^*\hat{O} \\ V^*\hat{O}
\end{array}
}\bigg\|_F^2 \\
&\leq \rbr{1 + \frac{2}{\sqrt{c-1}}}
\big\|Z - \eta \cdot \Pcal\rbr{\nabla G_Z(Z),S_Z} - Z^*\hat{O}\big\|_F^2,
\end{aligned}
\end{equation}
where the last inequality follows from Lemma~3.3 of
\cite{Li2016Stochastic}.  Therefore,
\begin{equation}
\label{eq:dz_bound}
d^2(Z^+, Z^*) \leq
 \rbr{1 + \frac{2}{\sqrt{c - 1}}}
\big[ d^2(Z, Z^*) - 2\eta \cdot (T_1+R_1) + 2\eta^2 \cdot (T_2+R_2) \big]
\end{equation}
where $T_1 = \dotp{\Pcal\rbr{\nabla f_Z(Z),S_Z}}{\Delta_Z}$, 
$T_2 = \big\|\sbr{\nabla f_Z(Z)}_{S_Z}\big\|_F^2$, $R_1 = \dotp{\Pcal\rbr{\nabla g_Z(Z),S_Z}}{\Delta_Z}$, and $R_2 = \big\|\sbr{\nabla g_Z(Z)}_{S_Z}\big\|_F^2$.

For the term $T_1$, we have
\begin{equation}
\begin{aligned}
\label{eq:T_1}
T_1 
& =    
\bdotp{\Pcal\rbr{\nabla f(UV^\top)V, S_U}}{\Delta_U} + 
\bdotp{\Pcal\rbr{\nabla f(UV^\top)^\top U, S_V}}{\Delta_V} \\
& =    
\underbrace{
\bdotp{\sbr{\nabla f(UV^\top) - \nabla f(U^*{V^*}^\top)}_{S_U,S_V}}{\sbr{UV^\top - U^*{V^*}^\top}_{S_U,S_V}}}_{T_{11}} \\
& \quad +
\underbrace{
\bdotp{\sbr{\nabla f(U^*{V^*}^\top)}_{S_U,S_V}}{\sbr{UV^\top - U^*{V^*}^\top}_{S_U,S_V}} 
}_{T_{12}} \\
& \quad + 
\underbrace{
\bdotp{\sbr{\nabla f(UV^\top)}_{S_U,S_V}}{\sbr{\Delta_U\Delta_V^\top}_{S_U,S_V}} 
}_{T_{13}}.
\end{aligned}
\end{equation}
Theorem 2.1.11 of
\cite{Nesterov2013Introductory} 
gives
\begin{equation}
\begin{aligned}
\label{eq:T_11}
T_{11} & \geq 
\frac{L \cdot\mu }{L  + \mu }\cdot\Big\|{UV^\top - U^*{V^*}^\top}\Big\|_F^2
 + \frac{1}{L +\mu }\cdot\Big\|\sbr{\nabla f(UV^\top) - \nabla f(U^*{V^*}^\top)}_{S_U,S_V}\Big\|_F^2 
\end{aligned}
\end{equation}
Next, we have
\begin{equation}
\begin{aligned}
\label{eq:T12}
T_{12} & \geq -
\abr{
\bdotp{\sbr{\nabla f(U^*{V^*}^\top)}_{S_U,S_V}}{\sbr{UV^\top - U^*{V^*}^\top}_{S_U,S_V}} 
} \\
& \stackrel{(i)}{\geq} - e_{\rm stat}\cdot \Big\|{UV^\top - U^*{V^*}^\top}\Big\|_F \\
& \stackrel{(ii)}{\geq}  - \frac{1}{2}\frac{L  + \mu }{L \cdot\mu }e_{\rm stat}^2 -
\frac{1}{2}\frac{L \cdot\mu }{L  + \mu }\cdot\Big\|{UV^\top - U^*{V^*}^\top}\Big\|_F^2
\end{aligned}
\end{equation}
where in $(i)$ follows from the definition of statistical error and in
$(ii)$ we used the Young's inequality
$ab \leq \frac{a^2}{2\epsilon} + \frac{\epsilon b^2}{2}$, for
$a,b,\epsilon > 0$. 
Therefore, 
\begin{equation}
\begin{aligned}
\label{eq:T_11_plus_T_12}
T_{11} + T_{12} & \geq
\frac{1}{2}\frac{L \cdot\mu }{L  + \mu }\cdot\Big\|{UV^\top - U^*{V^*}^\top}\Big\|_F^2
 + \frac{1}{L +\mu }\cdot\Big\|\sbr{\nabla f(UV^\top) - \nabla f(U^*{V^*}^\top)}_{S_U,S_V}\Big\|_F^2 \\
&\qquad - \frac{1}{2}\frac{L  + \mu }{L \cdot\mu }\cdot e_{\rm stat}^2.
\end{aligned}
\end{equation}
Finally, for the term $T_{13}$, we have 
\begin{equation}
\begin{aligned}
T_{13} 
& \geq - \abr{ \bdotp{\sbr{\nabla f(UV^\top)}_{S_U,S_V}}{\sbr{\Delta_U\Delta_V^\top}_{S_U,S_V}} } \\ 
& \geq - \abr{ \bdotp{\sbr{\nabla f(U^*{V^*}^\top)}_{S_U,S_V}}{\sbr{\Delta_U\Delta_V^\top}_{S_U,S_V}} } 
 - \abr{ \bdotp{\sbr{\nabla f(UV^\top) - \nabla f(U^*{V^*}^\top)}_{S_U,S_V}}{\sbr{\Delta_U\Delta_V^\top}_{S_U,S_V}} } \\ 
& \geq - \rbr{e_{\rm stat} + \Big\|\sbr{\nabla f(UV^\top) - \nabla f(U^*{V^*}^\top)}_{S_U,S_V}\Big\|_F}\cdot d^2(Z, Z^*),
\end{aligned}
\end{equation}
where the last inequality follows from the definition of statistical
error and the observation
$\|\Delta_U\Delta_V^\top\|_F \leq \|\Delta_V\|_F\cdot \|\Delta_U\|_F
\leq d^2(Z, Z^*)$.  Under the assumptions, 
\begin{equation}
d^2(Z, Z^*) \leq {\frac{4\mu_{\min}\sigma_r(\Theta^*)}{5(\mu +L )}}
\end{equation}
and therefore 
\begin{equation}
\begin{aligned}
\label{eq:T_13}
T_{13} 
& \geq - \rbr{e_{\rm stat} + \Big\|\sbr{\nabla f(UV^\top) - \nabla f(U^*{V^*}^\top)}_{S_U,S_V}\Big\|_F}\cdot
\sqrt{\frac{4\mu_{\min}\sigma_r(\Theta^*)}{5(\mu +L )}}
\cdot 
d(Z, Z^*) \\
& \geq -\frac{1}{2(\mu +L )}\cdot\rbr{e_{\rm stat}^2  
 + \Big\|\sbr{\nabla f(UV^\top) - \nabla f(U^*{V^*}^\top)}_{S_U,S_V}\Big\|_F^2} 
- \frac 45 \mu_{\min} \sigma_r(\Theta^*)
\cdot d^2(Z, Z^*).
\end{aligned}
\end{equation}
Combining \eqref{eq:T_11_plus_T_12} and \eqref{eq:T_13} we have 
\begin{equation}
\begin{aligned}
\label{eq:T_1_simplified}
T_{1} & \geq
\underbrace{\frac{1}{2}\frac{L \cdot\mu }{L  + \mu }\cdot\Big\|{UV^\top - U^*{V^*}^\top}\Big\|_F^2}_{T_{1a}}
- \frac 45 \mu_{\min} \sigma_r(\Theta^*) \cdot d^2(Z, Z^*) - \frac{1}{2}\rbr{\frac{L  + \mu }{L \cdot\mu }+\frac{1}{L +\mu }}\cdot e_{\rm stat}^2 \\
 &\qquad + \frac{1}{2(L +\mu )}\cdot\Big\|\sbr{\nabla f(UV^\top) - \nabla f(U^*{V^*}^\top)}_{S_U,S_V}\Big\|_F^2 .
\end{aligned}
\end{equation}

For the term $T_2$,  we have
\begin{equation}
\begin{aligned}
\Big\|[\nabla f(U^*{V^*}^\top)V]_{S_U}\Big\|_F &= \sup_{\|U_{S_U}\|_F=1} \tr\big(\nabla f(U^*{V^*}^\top)VU_{S_U}^\top\big) \\
&= \sup_{\|U_{S_U}\|_F=1} \langle \nabla f(U^*{V^*}^\top), U_{S_U}V^\top \rangle \\
&\leq e_{\text{stat}} \cdot \|V\|_2.
\end{aligned}
\end{equation} 
We then have
\begin{equation}
\begin{aligned}
\Big\|[\nabla f(UV^\top)V]_{S_U}\Big\|_F^2 &= \bignorm{[\nabla f(UV^\top)V-\nabla f(U^*{V^*}^\top)V+\nabla f(U^*{V^*}^\top)V]_{S_U}}_F^2 \\
&\leq 2\Big\|[\nabla f(UV^\top)V-\nabla f(U^*{V^*}^\top)V]_{S_U}\Big\|_F^2 + 2\Big\|[\nabla f(U^*{V^*}^\top)V]_{S_U}\Big\|_F^2  \\
&\leq 2\Big\|\sbr{\nabla f(UV^\top)-\nabla f(U^*{V^*}^\top)}_{S_U, S_V}\Big\|_F^2 \cdot \|V\|_2^2 + 2e_{\text{stat}}^2 \cdot \|V\|_2^2 \\
&\leq 2\bigg( \Big\|\sbr{\nabla f(UV^\top)-\nabla f(U^*{V^*}^\top)}_{S_U, S_V}\Big\|_F^2 + e_{\text{stat}}^2 \bigg) \cdot \|Z\|_2^2,
\end{aligned}
\end{equation} 
where the first inequality follows since
$\|A+B\|_F^2 \leq 2\|A\|_F^2 + 2\|B\|_F^2$, and the last inequality
follows since $\max(\|U\|_{2},\|V\|_{2}) \leq \|Z\|_{2}$. Combining the results, we have
\begin{equation}
\label{eq:T_2}
\begin{aligned}
T_2 &= 
\Big\|[\nabla f(UV^\top)V]_{S_U}\Big\|_F^2  +
\Big\|[\nabla f(UV^\top)^\top U]_{S_V}\Big\|_F^2 \\
& \leq 4\cdot
\rbr{
\Big\|\sbr{\nabla f(UV^\top)-\nabla f(U^*{V^*}^\top)}_{S_U, S_V}\Big\|_F^2 + 
e_{\text{stat}}^2
} \cdot \|Z\|_{2}^2.
\end{aligned}
\end{equation} 

For $R_1$, Lemma B.1 of \cite{park2016finding} gives
\begin{multline}
\label{eq:R_1}
R_1 \geq \underbrace{\frac 12 \|\nabla g\|_F^2}_{R_{11}} + \underbrace{\frac 18 \Big[\big\|UU^\top-U^*{U^*}^\top\big\|_F^2 + \big\|VV^\top-V^*{V^*}^\top\big\|_F^2 - 2\big\|UV^\top-U^*{V^*}^\top\big\|_F^2\Big]}_{R_{12}} \\
 - \underbrace{\frac 12 \|\nabla g\|_2 \cdot \|\Delta Z\|_F^2}_{R_{13}}.
\end{multline}
For $R_{12}$, we have that 
\begin{equation}
\begin{aligned}
\label{eq:R_12_plus_T_1a}
R_{12} + T_{1a} & = R_{12} + \frac{1}{8}\frac{L \cdot\mu }{L  + \mu }\cdot 4\Big\|{UV^\top - U^*{V^*}^\top}\Big\|_F^2 \\
&\geq \mu_{\min} \Big[\big\|UU^\top-U^*{U^*}^\top\big\|_F^2 + \big\|VV^\top-V^*{V^*}^\top\big\|_F^2 + 2\big\|UV^\top-U^*{V^*}^\top\big\|_F^2\Big] \\
&= \mu_{\min} \big\|ZZ^\top - Z^*{Z^*}^\top\big\|_F^2 \\
&\geq \frac 45 \mu_{\min} \sigma_r^2(Z^*) \cdot d^2(Z,Z^*) \\
&= \frac 85 \mu_{\min} \sigma_r(\Theta^*) \cdot d^2(Z,Z^*),
\end{aligned}
\end{equation}
where the first inequality follows from the definition of
$\mu_{\min}$, the second inequality follows from Lemma 5.4 of
\cite{Tu2016Low}, and the last equality follows from
$\sigma_r(Z^*) = \sqrt{2\sigma_r(\Theta^*)}$.

For $R_{13}$, recall that $\Delta Z$ satisfies \eqref{eq:bound_I_0}, we have that 
\begin{equation}
\begin{aligned}
\label{eq:R_13}
R_{13} & \leq \frac 12 \|\nabla g\|_2 \cdot \|\Delta Z\|_F \cdot \sqrt{ \frac{8}{5} \mu_{\min}\sigma_r(\Theta^*) }\\
&\leq \frac{2}{5} \mu_{\min}\sigma_r(\Theta^*) \cdot d^2(Z,Z^*) + \frac 14 \|\nabla g\|_F^2.
\end{aligned}
\end{equation}
Combining \eqref{eq:T_1_simplified}, \eqref{eq:R_1}, \eqref{eq:R_12_plus_T_1a}, and \eqref{eq:R_13}, we obtain
\begin{equation}
\begin{aligned}
\label{eq:T_1_plus_R_1}
T_1 + R_1 &\geq \frac{2}{5} \mu_{\min}\sigma_r(\Theta^*) \cdot d^2(Z,Z^*) + \frac 14 \|\nabla g\|_F^2
- \frac{1}{2}\rbr{\frac{L  + \mu }{L \cdot\mu }+\frac{1}{L +\mu }}\cdot e_{\rm stat}^2  \\
&\qquad + \frac{1}{2(L +\mu )}\cdot\Big\|\sbr{\nabla f(UV^\top) - \nabla f(U^*{V^*}^\top)}_{S_U,S_V}\Big\|_F^2 .
\end{aligned}
\end{equation}

For $R_2$, we have 
\begin{equation}
\begin{aligned}
\label{eq:R_2}
R_2 = \|U \nabla g\|_F^2 + \|V \nabla g\|_F^2 \leq 2\|Z\|_2^2 \cdot \|\nabla g\|_F^2.
\end{aligned}
\end{equation}
Combining \eqref{eq:T_2}, \eqref{eq:T_1_plus_R_1}, and \eqref{eq:R_2}, we have
\begin{equation}
\begin{aligned}
\label{eq:d2_intermediate}
d^2(Z, Z^*) - & 2\eta \cdot (T_1+R_1) + 2\eta^2 \cdot (T_2+R_2)  \\
& \leq \rbr{1 - \eta\cdot
\frac{2}{5} \mu_{\min}\sigma_r(\Theta^*)
}\cdot d^2(Z, Z^*) \\
& \qquad +
\eta\rbr{4\eta\cdot\|Z\|_2^2 - \frac{1}{2(L +\mu )}}\cdot\Big\|\sbr{\nabla f(UV^\top) - \nabla f(U^*{V^*}^\top)}_{S_U,S_V}\Big\|_F^2 \\
&\qquad + \eta\bigg(2\eta\cdot\|Z\|_2^2 - \frac 14 \bigg) \|\nabla g\|_F^2 \\
&\qquad + 
\eta
\rbr{
\frac{L  + \mu }{2\mu L }+\frac{1}{2(L +\mu )} +
4\eta\cdot\|Z\|_2^2
}\cdot e_{\rm stat}^2.
\end{aligned}
\end{equation}
Under the choice of the step size,
\begin{equation}
\label{eq:choice_of_step_size}
\eta \leq \frac{1}{8\|Z\|_2^2} \cdot \min\Big\{\frac{1}{2(\mu +L )}, 1\Big\},
\end{equation}
the second term and third term in \eqref{eq:d2_intermediate} are non-positive and we
drop them to  get 
\begin{equation}
\begin{aligned}
\label{eq:d2_intermediate_simplified}
d^2(Z, Z^*) - & 2\eta \cdot (T_1+R_1) + 2\eta^2 \cdot (T_2+R_2)   \\
& \leq \rbr{1 - \eta\cdot
\frac{2}{5} \mu_{\min}\sigma_r(\Theta^*)
}\cdot d^2(Z, Z^*)
 +
\eta\cdot\frac{L  + \mu }{L \cdot\mu }\cdot e_{\rm stat}^2.
\end{aligned}
\end{equation}
Plugging \eqref{eq:d2_intermediate_simplified} into
\eqref{eq:dz_bound} we finish the proof.

\subsection{Proof of Lemma \ref{lemma:step_size_constant}}
\label{A:lemma:step_size_constant}

Comparing \eqref{eq:eta_constant} and \eqref{eq:step_size_condition} we see that we only need to show $\|Z\|_2^2 \leq 2\|Z_0\|_2^2$. 
Let $O \in \Ocal(r)$ be such that
\[
d^2(Z, Z^*) = \|U - U^* O \|_F^2 + \|V - V^* O \|_F^2.
\]
By triangular inequality we have
\begin{equation}
\begin{aligned}
\label{eq:98}
\|Z\|_2 &\leq \|Z^*O\|_2 + \|Z - Z^*O\|_2 \\
&\leq \|Z^*\|_2 + \sqrt{\frac 45 \mu_{\min}\sigma_r(\Theta^*) \cdot \frac{1}{\mu +L }} \\
&\leq \|Z^*\|_2 + \sqrt{\frac 45 \cdot \frac 18 \frac{\mu L }{\mu +L } \cdot \frac 12\sigma_r^2(Z^*) \cdot \frac{1}{\mu +L }} \\
&\leq \|Z^*\|_2 + \sqrt{\frac{1}{80} \sigma_r^2(Z^*) } \\
&\leq \frac 98\|Z^*\|_2,
\end{aligned}
\end{equation}
where the third inequality follows from the definition of $\mu_{\min}$
and $\sigma_r^2(Z^*) = 2\sigma_r(\Theta^*)$, and the fourth inequality
follows from $\frac{ab}{(a+b)^2} \leq \frac 14$. Similarly, we have
\begin{equation}
\begin{aligned}
\label{eq:78}
\|Z_0\|_2 &\geq \|Z^*O\|_2 - \|Z_0 - Z^*O\|_2 \\
&\geq \|Z^*\|_2 - \sqrt{\frac{1}{80} \sigma_r^2(Z^*) } \\
&\geq \frac 78\|Z^*\|_2.
\end{aligned}
\end{equation}
Combining \eqref{eq:98} and \eqref{eq:78} we have 
\begin{equation}
\begin{aligned}
\|Z\|_2 \leq \frac98 \cdot \frac 87 \|Z_0\|_2 \leq \sqrt{2} \|Z_0\|_2,
\end{aligned}
\end{equation}
which completes the proof.

\subsection{Proof of Lemma \ref{lemma:e_stat}}
\label{A:lemma:e_stat}

Let $\Omega(s, m)$ denote a collection of subsets of
$\cbr{1,\ldots,m}$ of size $s$.  Let $S_U \in \Omega(s_1, p)$ and
$S_V \in \Omega(s_2, k)$ be fixed. With some abuse of notation, let
$\Wcal(S_U) = \{U \in \RR^{p \times 2r} \mid \|U_{S_U^c}\| = 0,
\|U_{S_U}\|_2 = 1\}$ and 
$\Wcal(S_V) = \{V \in \RR^{k \times 2r} \mid \|V_{S_V^c}\| = 0,
\|V_{S_V}\|_F = 1\}$.
Let $\Ncal_U(\epsilon)$ and $\Ncal_V(\epsilon)$ be the epsilon net of
$\Wcal_U$ and $\Wcal_V$, respectively.  Using Lemma 10 and Lemma 11 of
\cite{Vu2011Singular}, we know that
$|\Ncal_U(\epsilon)| \leq (3\epsilon^{-1})^{2r\cdot s_1}$,
$|\Ncal_V(\epsilon)| \leq (3\epsilon^{-1})^{2r\cdot s_2}$, and
\begin{equation}
\begin{aligned}
\sup_{\substack{U\in\Wcal(S_U)\\V\in\Wcal(S_V)}} \frac 1n \tr\big(E^\top XUV^\top\big) 
&\leq (1-\epsilon)^{-2} \max_{\substack{U \in \Ncal_U(\epsilon) \\ V \in \Ncal_V(\epsilon)}} \frac 1n \tr\big(E^\top XUV^\top\big).
\end{aligned}
\end{equation}
For fixed $U$ and $V$, the random variable
$\tr\big(E^\top XUV^\top\big)$ is a sub-Gaussian
with variance proxy $\sigma^2\|  X_{S_U}U_{S_U}V_{S_V}^\top \|_F^2$.
This variance proxy can be bounded as 
\[
\sigma^2\|  X_{S_U}U_{S_U}V_{S_V}^\top \|_F^2 \leq 
\sigma^2\cdot \max_{S_U \in \Omega(s_1, p)}\|(X^\top X)_{S_US_U}\|_2
= n \sigma^2 \bar\kappa(s_1).
\] 
Using a tail bound for sub-Gaussian random variables, 
we get
\[
\frac 1n \tr\big(E^\top XU_{S_U}V_{S_V}^\top\big)  \leq
2\sigma \sqrt{ \frac{\bar\kappa(s_1)\log\frac1\delta}{n} }
\]
with probability at least $1-\delta$.
To obtain an upper bound on $e_{\text{stat}}$, we will apply the union bound
$\Omega(s_1, p)$, $\Omega(s_2, k)$, $\Ncal_U(\epsilon)$ and $\Ncal_V(\epsilon)$.
We set $\epsilon=\frac12$ and obtain
\[
e_{\text{stat}}
\leq
8\sigma \sqrt{\frac{\bar\kappa(s_1)}{n} 
\Big(s_1\log p + s_2\log k + 2r (s_1+s_2)\log6 + \log \frac1\delta }\Big)
\]
with probability at least $1-\delta$. Taking
$\delta = (p \vee k)^{-1}$ completes the proof.
